\documentclass{article}

\PassOptionsToPackage{numbers, compress}{natbib}

\usepackage[final]{neurips_2020}

\usepackage[utf8]{inputenc} 
\usepackage[T1]{fontenc}    
\usepackage{hyperref}       
\usepackage{url}            
\usepackage{booktabs}       
\usepackage{amsfonts}       
\usepackage{nicefrac}       
\usepackage{microtype}      

\usepackage{mysymbols}
\usepackage{algorithm}
\usepackage{algpseudocode}
\renewcommand{\algorithmiccomment}[1]{\textit{#1}}
\usepackage{amsmath}
\usepackage{amsthm}
\usepackage{amssymb}
\usepackage{stmaryrd}
\usepackage{latexsym}
\usepackage{bm}
\usepackage{xcolor}
\usepackage{subcaption}
\usepackage{placeins}
\theoremstyle{plain}
\usepackage{mathtools}
\newcommand{\defeq}{\vcentcolon=}
\usepackage{thmtools,thm-restate}
\usepackage{cleveref}
\usepackage{wrapfig}
\usepackage{enumitem}
\usepackage[normalem]{ulem}
\newcommand{\stkout}[1]{\ifmmode\textcolor{red}{\text{\sout{\ensuremath{#1}}}}\else\sout{#1}\fi}

\DeclareMathOperator*{\argmax}{argmax}

\usepackage{tikz}
\usetikzlibrary{automata, positioning, arrows}

\tikzset{
    ->, 
    >=stealth,
}

\definecolor{bettergreen}{HTML}{4faa42}
\definecolor{mygreen}{HTML}{91ed92}
\definecolor{myred}{HTML}{f28d8d}
\definecolor{mygrey}{HTML}{a1a1a1}

\title{Learning to Incentivize Other Learning Agents}

\author{
    Jiachen Yang\textsuperscript{\normalfont 1}\thanks{Work done during internship at DeepMind},
    Ang Li\textsuperscript{\normalfont 2},
    Mehrdad Farajtabar\textsuperscript{\normalfont 2},
    Peter Sunehag\textsuperscript{\normalfont 2},
    Edward Hughes\textsuperscript{\normalfont 2},
    Hongyuan Zha\textsuperscript{\normalfont 1,3}\thanks{On leave from College of Computing, Georgia Institute of Technology} \\
    \textsuperscript{1}Georgia Institute of Technology \quad
    \textsuperscript{2}DeepMind\\
    \textsuperscript{3}AIRS and Chinese University of Hong Kong, Shenzhen\\
    \texttt{jiachen.yang@gatech.edu}\\
    \texttt{[anglili,farajtabar,sunehag,edwardhughes]@google.com}\\
    \texttt{zha@cc.gatech.edu}
}

\begin{document}

\maketitle

\begin{abstract}
The challenge of developing powerful and general Reinforcement Learning (RL) agents has received increasing attention in recent years.
Much of this effort has focused on the single-agent setting, in which an agent maximizes a predefined extrinsic reward function.
However, a long-term question inevitably arises: how will such independent agents cooperate when they are continually learning and acting in a shared multi-agent environment?
Observing that humans often provide incentives to influence others' behavior, we propose to equip each RL agent in a multi-agent environment with the ability to give rewards directly to other agents, using a learned incentive function.
Each agent learns its own incentive function by explicitly accounting for its impact on the learning of recipients and, through them, the impact on its own extrinsic objective.
We demonstrate in experiments that such agents significantly outperform standard RL and opponent-shaping agents in challenging general-sum Markov games, often by finding a near-optimal division of labor.
Our work points toward more opportunities and challenges along the path to ensure the common good in a multi-agent future.
\end{abstract}


\section{Introduction}
\label{sec:intro}
Reinforcement Learning (RL) \citep{sutton2018reinforcement} agents are achieving increasing success on an expanding set of tasks \citep{mnih2015human,jaderberg2019human,silver2017mastering,vinyals2019grandmaster,berner2019dota}.
While much effort is devoted to single-agent environments and fully-cooperative games, there is a possible future in which large numbers of RL agents with imperfectly-aligned objectives must interact and continually learn in a shared multi-agent environment. 
The option of centralized training with a global reward \citep{foerster2018counterfactual,sunehag2018value,rashid2018a} is excluded as it does not scale easily to large populations and may not be adopted by self-interested parties.
On the other hand, the paradigm of decentralized training---in which no agent is designed with an objective to maximize collective performance and each agent optimizes its own set of policy parameters---poses difficulties for agents to attain high individual and collective return \citep{olson1965logic}. 
In particular, agents in many real world situations with mixed motives, such as settings with nonexcludable and subtractive common-pool resources, may face a social dilemma wherein mutual selfish behavior leads to low individual and total utility, due to fear of being exploited or greed to exploit others \citep{rapoport1974prisoner,leibo2017multi,lerer2017maintaining}.
Whether, and how, independent learning and acting agents can cooperate while optimizing their own objectives is an open question.


The conundrum of attaining multi-agent cooperation with decentralized
training of agents, who may have misaligned individual objectives, requires us to go beyond the restrictive mindset that the collection of predefined individual rewards cannot be changed by the agents themselves.
We draw inspiration from the observation that this fundamental multi-agent problem arises at multiple scales of human activity and, crucially, that it can be successfully resolved when agents give the right incentives to \textit{alter} the objective of other agents, in such a way that the recipients' behavior changes for everyone's advantage. 
Indeed, a significant amount of individual, group, and international effort is expended on creating effective incentives or sanctions to shape the behavior of other individuals, social groups, and nations \citep{veroff2016social,delfgaauw2008incentives,doxey1980economic}. 
The rich body of work on game-theoretic side payments \citep{jackson2005endogenous,harstad2008side,fong2009optimal} further attests to the importance of inter-agent incentivization in society.

Translated to the framework of Markov games for multi-agent reinforcement learning (MARL) \citep{littman1994markov}, the key insight is to remove the constraints of an immutable reward function.
Instead, we allow agents to \textit{learn} an incentive function that gives rewards to other learning agents and thereby shape their behavior.
The new learning problem for an agent 
becomes two-fold: learn a policy that optimizes the total extrinsic rewards and incentives it receives, and learn an incentive function that alters other agents' behavior so as to optimize its own extrinsic objective.
While the emergence of incentives in nature may have an evolutionary explanation \citep{guth1995evolutionary}, human societies contain ubiquitous examples of learned incentivization and we focus on the learning viewpoint in this work.

{\bf The Escape Room game.}
We may illustrate the benefits and necessity of incentivization with a simple example. 
The \textit{Escape Room} game $\text{ER}(N,M)$ is a discrete $N$-player Markov game with individual extrinsic rewards and parameter $M < N$, as shown in \Cref{fig:symmetric-room}.
An agent gets +10 extrinsic reward for exiting a door and ending the game, but the door can only be opened when $M$ other agents cooperate to pull the lever.
However, an extrinsic penalty of $-1$ for any movement discourages all agents from taking the cooperative action.
If agents optimize their own rewards with standard independent RL, no agent can attain positive reward, as we show in \Cref{sec:results}.

This game may be solved by equipping agents with the ability to incentivize other agents to pull the lever.
However, we hypothesize---and confirm in experiments---that 
merely augmenting an agent's action space with a ``give-reward'' action and applying standard RL faces significant learning difficulties.
Consider the case of $\text{ER}(2,1)$: suppose we allow agent A1 an additional action that sends +2 reward to agent A2, and let it observe A2's chosen action prior to taking its own action.
Assuming that A2 conducts sufficient exploration, an intelligent reward-giver should learn to use the give-reward action to incentivize A2 to pull the lever.
However, RL optimizes the expected cumulative reward within \textit{one} episode,
but the effect of a give-reward action manifests in the recipient's behavior only after many learning updates that generally span \textit{multiple} episodes.
Hence, a reward-giver may not receive any feedback within an episode, much less an immediate feedback, on whether the give-reward action benefited its own extrinsic objective.
Instead, we need an agent that explicitly accounts for the impact of incentives on the recipient's learning and, thereby, on its own future performance.


As a first step toward addressing these new challenges, we make the following conceptual, algorithmic, and experimental contributions.
(1) We create an agent that learns an incentive function to reward other learning agents, by explicitly accounting for the impact of incentives on its own performance, through the learning of recipients.
(2) Working with agents who conduct policy optimization, we derive the gradient of an agent's extrinsic objective with respect to the parameters of its incentive function.
We propose an effective training procedure based on online cross-validation to update the incentive function and policy on the same time scale.
(3) We show convergence to mutual cooperation in a matrix game, and experiment on a new deceptively simple \textit{Escape Room} game, which poses significant difficulties for standard RL and action-based opponent-shaping agents, but on which our agent consistently attains the global optimum.
(4) Finally, our agents discover near-optimal division of labor in the challenging and high-dimensional social dilemma problem of \textit{Cleanup} \citep{hughes2018inequity}.
Taken together, we believe this is a promising step toward a cooperative multi-agent future.

\begin{figure}[t]
    \centering
    \begin{tikzpicture}
        \node[state, fill=mygrey] at (0,0) (s1) {start};
        \node[state, left of=s1, fill=myred] at (-1.5,0) (s0) {lever};
        \node[state, right of=s1, fill=mygreen] at (1.5,0) (s2) {door};
        \draw 
            (s0) edge[bend left, below] node{-1} (s1)
            (s1) edge[bend left, above] node{-1} (s0)
            (s1) edge[bend left, below] node{-1 or 10} (s2)
            (s2) edge[bend left, above] node{-1} (s1)
            (s0) edge[bend left, above] node{-1 or 10} (s2)
            (s2) edge[bend left, above] node{-1} (s0)
        ;
    \end{tikzpicture}
    \caption{The $N$-player \textit{Escape Room} game $\text{ER}(N,M)$. For $M < N$, if fewer than $M$ agents pull the lever, which incurs a cost of $-1$, then all agents receive $-1$ for changing positions. Otherwise, the agent(s) who is not pulling the lever can get $+10$ at the door and end the episode.}
    \label{fig:symmetric-room}
\end{figure}
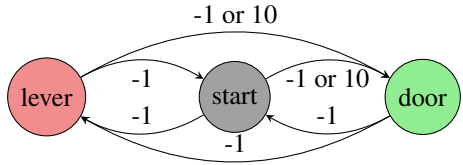

\vspace{-2mm}
\section{Related work}
\label{sec:related-work}
\vspace{-1mm}
Learning to incentivize other learning agents is motivated by the problem of cooperation among independent learning agents in intertemporal social dilemmas (ISDs) \citep{leibo2017multi}, in which defection is preferable to individuals in the short term but mutual defection leads to low collective performance in the long term.
Algorithms for fully-cooperative MARL \citep{foerster2018counterfactual,rashid2018a,sunehag2018value} may not be applied as ISDs have mixed motives and cannot canonically be reduced to fully cooperative problems.
Previous work showed that collective performance can be improved by independent agents with \textit{intrinsic} rewards  \citep{eccles2019learning,hughes2018inequity,wang2019evolving,jaques2019social,hostallero2020inducing}, which
are either hand-crafted or slowly evolved based on other agents' performance and modulate each agent's own total reward.
In contrast, a reward-giver's incentive function in our work is \textit{learned} on the same timescale as policy learning and is given to, and maximized by, \textit{other} agents.
Empirical research shows that augmenting an agent's action space with a ``give-reward'' \textit{action} can improve cooperation during certain training phases in ISDs \citep{lupu2020gifting}.

Learning to incentivize is a form of opponent shaping, whereby an agent learns to influence the learning update of other agents for its own benefit. While LOLA \citep{foerster2018learning} and SOS \citep{letcher2019stable} exert influence via actions taken by its policy, whose effects manifest through the Markov game state transition, our proposed agent exerts direct influence via an incentive function, which is distinct from its policy and which explicitly affects the recipient agent's learning update. Hence the need to influence other agents does not restrict a reward-giver's policy, potentially allowing for more flexible and stable shaping. We describe the mathematical differences between our method and LOLA in Section \ref{subsec:diff-from-lola}, and experimentally compare with LOLA agents augmented with reward-giving actions.


Our work is related to a growing collection of work on modifying or learning a reward function that is in turn maximized by another learning algorithm \citep{baumann2018adaptive,sodomka2013coco,zheng2018learning}.
Previous work investigate the evolution of the prisoner's dilemma payoff matrix when altered by a ``mutant'' player who gives a fixed incentive for opponent cooperation \citep{akccay2011evolution};
employ a centralized operator on utilities in 2-player games with side payments \citep{sodomka2013coco};
and directly optimize collective performance by centralized rewarding in 2-player matrix games \citep{baumann2018adaptive}.
In contrast, we work with $N$-player Markov games with self-interested agents who must individually learn to incentivize other agents and cannot optimize collective performance directly.
Our technical approach is inspired by online cross validation \citep{sutton1992adapting}, which is used to optimize hyperparameters in meta-gradient RL \citep{xu2018meta}, and by the optimal reward framework \citep{singh2009rewards}, in which a single agent learns an intrinsic reward by ascending the gradient of its own extrinsic objective \citep{zheng2018learning}.





\section{Learning to incentivize others}
We design Learning to Incentivize Others (LIO), an agent that learns an incentive function by explicitly accounting for its impact on recipients' behavior, and through them, the impact on its own extrinsic objective.
For clarity, we describe the ideal case where agents have a perfect model of other agents' parameters and gradients; afterwards, we remove this assumption via opponent modeling.
We present the general case of $N$ LIO agents, indexed by $i \in [N] := \lbrace 1, \dotsc, N \rbrace$.
Each agent gives rewards using its incentive function and learns a regular policy with all received rewards.
For clarity, we use index $i$ when referring to the reward-giving part of an agent, and we use $j$ for the part that learns from received rewards.
For each agent $i$, let $o^i \defeq O^i(s) \in \Ocal$ denote its individual observation at global state $s$; $a^i \in \Acal^i$ its action; and $-i$ a collection of all indices except $i$.
Let $\abf$ and $\pibf$ denote the joint action and the joint policy over all agents, respectively.

A reward-giver agent $i$ learns a vector-valued incentive function $r_{\eta^i} \colon \Ocal \times \Acal^{-i} \mapsto \Rbb^{N-1}$, parameterized by $\eta^i \in \Rbb^n$, that maps its own observation $o^i$ and all other agents' actions $a^{-i}$ to a vector of rewards for the other $N-1$ agents\footnote{We do not allow LIO to reward itself, as our focus is on influencing \textit{other} agents' behavior. Nonetheless, LIO may be complemented by other methods for learning \textit{intrinsic} rewards \citep{zheng2018learning}.}.
Let $r_{\eta^i}^j$ denote the reward that agent $i$ gives to agent $j$.
As we elaborate below, $r_{\eta^i}$ is separate from the agent's conventional policy and is learned via direct gradient ascent on the agent's own extrinsic objective, involving its effect on all other agents' policies, instead of via RL.
Therefore, while it may appear that LIO has an augmented action space that provides an additional channel of influence on other agents, we emphasize that LIO's learning approach does \textit{not} treat the incentive as a standard ``give-reward'' action.

We build on the idea of online cross-validation \citep{sutton1992adapting}, to capture the fact that an incentive has measurable effect only after a recipient's learning step.
As such, we describe LIO in a procedural manner below (\Cref{alg:lio}).
This procedure can also be viewed as an iterative method for a bilevel optimization problem \citep{colson2007overview}, where the upper level optimizes the incentive function by accounting for recipients' policy optimization at the lower level.
At each time step $t$, each recipient $j$ receives a total reward
\begin{align}\label{eq:recipient-reward}
    r^j(s_t,\abf_t, \eta^{-j}) &\defeq r^{j,\text{env}}(s_t, \abf_t) + \sum_{i \neq j} r_{\eta^i}^j(o^i_t, a^{-i}_t) \, ,
\end{align}
where $r^{j,\text{env}}$ denotes agent $j$'s extrinsic reward.
Each agent $j$ learns a standard policy $\pi^j$, parameterized by $\theta^j \in \Rbb^m$, to maximize the objective
\begin{align}\label{eq:policy-objective}
    \max_{\theta^j} J^{\text{policy}}(\theta^j, \eta^{-j}) \defeq \Ebb_{\pibf}\left[ \sum_{t=0}^{T} \gamma^t r^j(s_t,\abf_t, \eta^{-j}) \right] \, .
\end{align}
Upon experiencing a trajectory $\tau^j \defeq (s_0, \abf_0, r^j_0, \dotsc, s_T)$, the recipient carries out an update
\begin{align}\label{eq:recipient-update}
    \hat{\theta}^j \leftarrow \theta^j + \beta f(\tau^j, \theta^j, \eta^{-j} )
\end{align}
that adjusts its policy parameters with learning rate $\beta$ (\Cref{alg:lio}, lines 4-5).
Assuming policy optimization learners in this work and choosing policy gradient for exposition, the update function is
\begin{align}\label{eq:recipient-gradient}
    f(\tau^j, \theta^j, \eta^{-j}) = \sum_{t=0}^{T} \nabla_{\theta^j} \log \pi^j(a^j_t|o^j_t)G^j_t(\tau^j; \eta^{-j}) \, ,
\end{align}
where the return $G^j_t(\tau^j, \eta^{-j}) = \sum_{l=t}^T \gamma^{l-t}r^j(s_l,\abf_l, \eta^{-j})$ depends on incentive parameters $\eta^{-j}$.

After each agent has updated its policy to $\hat{\pi}^j$, parameterized by new $\hat{\theta}^j$, it generates a new trajectory $\hat{\tau}^j$. 
Using these trajectories, each reward-giver $i$ updates its individual incentive function parameters $\eta^i$ to maximize the following individual objective (\Cref{alg:lio}, lines 6-7):
\begin{align}\label{eq:reward-objective}
    \max_{\eta^i} J^i(\hat{\tau}^i, \tau^i, \hat{\thetabf}, \eta^i) &\defeq \Ebb_{\hat{\pibf}}\left[ \sum_{t=0}^{T} \gamma^t \hat{r}^{i,\text{env}}_t \right] - \alpha L(\eta^i,\tau^i) \, .
\end{align}
The first term is the expected extrinsic return of the reward-giver in the new trajectory $\hat{\tau}^i$.
It implements the idea that the purpose of agent $i$'s incentive function is to alter other agents' behavior so as to maximize its extrinsic rewards.
The rewards it received from others are already accounted by its own policy update.
The second term is a cost for giving rewards in the first trajectory $\tau^i$:
\begin{align}\label{eq:reward-regularizer}
    L(\eta^i, \tau^i) \defeq \sum_{(o_t^i, a_t^{-i}) \in \tau^i} \gamma^t \lVert r_{\eta^i}(o^i_t,a^{-i}_t) \rVert_1 \, .
\end{align}
This cost is incurred by the incentive function and not by the policy, since the latter does not determine incentivization\footnote{Note that the 
outputs of the incentive function and policy
are conditionally independent given the agent's observation, but their separate learning processes are coupled via the learning process of other agents.} and should not be penalized for the incentive function's behavior (see \Cref{app:cost} for more discussion).
We use the $\ell_1$-norm so that cost has the same physical ``units'' as extrinsic rewards.
The gradient of \eqref{eq:reward-regularizer} is directly available, assuming $r_{\eta^i}$ is a known function approximator (e.g., neural network).
Letting $J^i(\hat{\tau}^i,\hat{\thetabf})$ denote the first term in \eqref{eq:reward-objective}, the gradient w.r.t. $\eta^i$ is:
\begin{align}\label{eq:chain-rule}
    \nabla_{\eta^i} J^i(\hat{\tau}^i,\hat{\thetabf}) &= \sum_{j\neq i} (\nabla_{\eta^i} \hat{\theta}^j)^T \nabla_{\hat{\theta}^j} J^i(\hat{\tau}^i,\hat{\thetabf}) \, .
\end{align}
The first factor of each term in the summation follows directly from \eqref{eq:recipient-update} and \eqref{eq:recipient-gradient}:
\begin{align}
    \nabla_{\eta^i}\hat{\theta}^j = \beta \sum_{t=0}^{T} \nabla_{\theta^j} \log \pi^j(a^j_t|o^j_t)\left(\nabla_{\eta^i} G^j_t(\tau^j;\eta^{-j}) \right)^T \, .
\end{align}
Note that \eqref{eq:recipient-update} does not contain recursive dependence of $\theta^j$ on $\eta^i$ since $\theta^j$ is a function of incentives in \textit{previous} episodes, not those in trajectory $\tau^i$.
The second factor in \eqref{eq:chain-rule} can be derived as 
\begin{align}\label{eq:chain-rule-factor}
    \nabla_{\hat{\theta}^j} J^i(\hat{\tau}^i,\hat{\thetabf}) = \Ebb_{\hat{\pibf}}\left[ \nabla_{\hat{\theta}^j} \log \hat{\pi}^j(\ahat^j|\ohat^j)Q^{i,\hat{\pibf}}(\hat{s}, \hat{\abf}) \right] \, .
\end{align}
In practice, to avoid manually computing the matrix-vector product in \eqref{eq:chain-rule}, one can define the loss
\begin{align}\label{eq:eta-loss}
    \text{Loss}(\eta^i, \hat{\tau}^i) \defeq - \sum_{j \neq i}\sum_{t=0}^T \log \pi_{\hat{\theta}^j}(\hat{a}^j_t|\hat{o}^j_t) \sum_{l=t}^T \gamma^{l-t}r^{i, \text{env}}(\hat{s}_l,\hat{\abf}_l) \, ,
\end{align}
and directly minimize it via automatic differentiation \citep{abadi2016tensorflow}.
Crucially, $\hat{\theta}^j$ must preserve the functional dependence of the policy update step \eqref{eq:recipient-gradient} on $\eta^i$ within the same computation graph.
Derivations of \eqref{eq:chain-rule-factor} and \eqref{eq:eta-loss} are similar to that for policy gradients \citep{sutton2000policy} and are provided in \Cref{app:derivation}.


\begin{algorithm}[t]
\caption{Learning to Incentivize Others}
\label{alg:lio}
\begin{algorithmic}[1]
\Procedure{Train LIO agents}{}
\State Initialize all agents' policy parameters $\theta^i$, incentive function parameters $\eta^i$ 
\For{each iteration}
    \State Generate a trajectory $\lbrace \tau^j \rbrace$ using $\thetabf$ and $\etabf$
    \State For all reward-recipients $j$, update $\hat{\theta}^j$ using \eqref{eq:recipient-update}
    \State Generate a new trajectory $\lbrace \hat{\tau}^i \rbrace$ using new $\hat{\thetabf}$
    \State For reward-givers $i$, compute new $\hat{\eta}^i$ by gradient ascent on \eqref{eq:reward-objective}
    \State $\theta^i \leftarrow \hat{\theta}^i$, $\eta^i \leftarrow \hat{\eta}^i$ for all $i \in [N]$.
\EndFor
\EndProcedure
\end{algorithmic}
\end{algorithm}
LIO is compatible with the goal of achieving emergent cooperation in fully-decentralized MARL, as
agents already learn individual sets of parameters to maximize individual objectives.
One may directly apply opponent modeling \citep{albrecht2018autonomous} when
LIO can observe, or estimate, other agents' egocentric observations, actions, and individual rewards, and have common knowledge that all agents conduct policy updates via reinforcement learning.
These requirements are satisfied in environments where incentivization itself is feasible, since these observations are required for rational incentivization.
LIO may then fit a behavior model for each opponent,
create an internal model of other agents' RL processes, and learn the incentive function by differentiating through fictitious updates using the model in place of \eqref{eq:recipient-update}.
We demonstrate a fully-decentralized implementation in our experiments.

\vspace{-1mm}
\subsection{Relation to opponent shaping via actions}
\label{subsec:diff-from-lola}
\vspace{-1mm}
LIO conducts opponent shaping via the incentive function.
This resembles LOLA \citep{foerster2018learning}, but there are key algorithmic differences.
Firstly, LIO's incentive function is trained separately from its policy parameters, while opponent shaping in LOLA depends solely on the policy.
Secondly, the LOLA gradient correction for agent $i$ is derived from $\nabla_{\theta^i} J^i(\theta^i, \theta^j + \Delta \theta^j)$ under Taylor expansion, but LOLA disregards a term with $\nabla_{\theta^i} \nabla_{\theta^j} J^i(\theta^i,\theta^j)$ even though it is non-zero in general.
In contrast, LIO is constructed from the principle of online cross-validation \citep{sutton1992adapting}, not Taylor expansion, and hence this particular mixed derivative is absent---the analogue for LIO would be $\nabla_{\eta^i}\nabla_{\theta^j} J^i(\theta^i, \theta^j)$, which is zero because incentive parameters $\eta^i$ affect all agents \textit{except} agent $i$.
Thirdly, LOLA optimizes its objective assuming one step of opponent learning, \textit{before} the opponent actually does so \citep{letcher2019stable}.
In contrast, LIO updates the incentive function \textit{after} recipients carry out policy updates using received incentives.
This gives LIO a more accurate measurement of the impact of incentives, which reduces variance and increases performance, as we demonstrate experimentally in \Cref{app:asymmetric} by comparing with a 1-episode variant of LIO that does not wait for opponent updates.
Finally, by adding differentiable reward channels to the environment, which is feasible in many settings with side payments \citep{jackson2005endogenous}, LIO is closer in spirit to the paradigm of optimized rewards \citep{singh2009rewards,jaderberg2019human,wang2019evolving}.

\FloatBarrier

\begin{figure}[t]
\centering
\begin{subfigure}[t]{0.24\linewidth}
    \centering
    \includegraphics[width=1.0\linewidth]{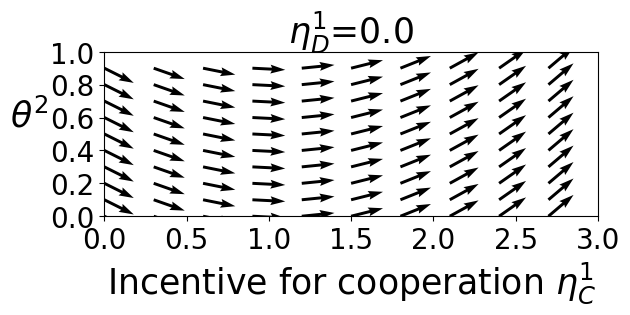}
    \label{fig:vf-t2-e1c-1}
\end{subfigure}
\hfill
\begin{subfigure}[t]{0.24\linewidth}
    \centering
    \includegraphics[width=1.0\linewidth]{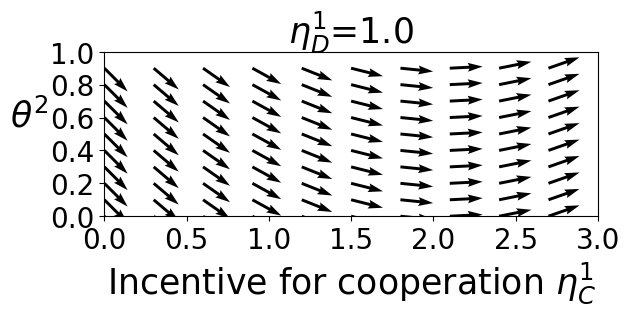}
    \label{fig:vf-t2-e1c-2}
\end{subfigure}
\hfill
\begin{subfigure}[t]{0.24\linewidth}
    \centering
    \includegraphics[width=1.0\linewidth]{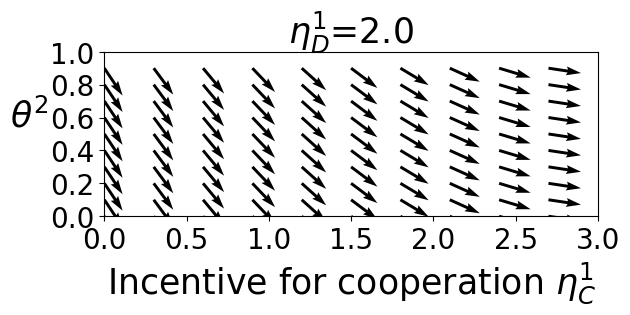}
    \label{fig:vf-t2-e1c-3}
\end{subfigure}
\hfill
\begin{subfigure}[t]{0.24\linewidth}
    \centering
    \includegraphics[width=1.0\linewidth]{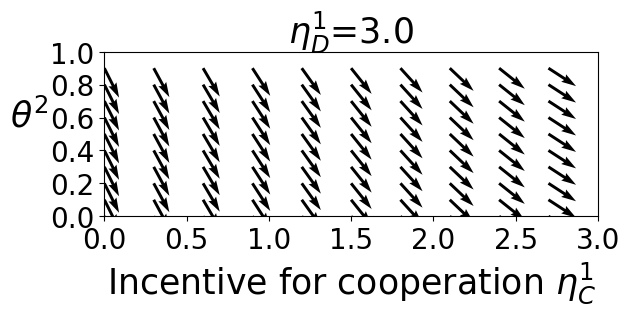}
    \label{fig:vf-t2-e1c-4}
\end{subfigure}

\vspace{-2mm}
\caption{Exact LIO in IPD: probability of recipient cooperation versus incentive for cooperation. 
}
\vspace{-2mm}
\label{fig:vf-t2-e1}
\end{figure}

\vspace{-2mm}
\subsection{Analysis in Iterated Prisoner's Dilemma}
\label{subsec:ipd-analysis}
\vspace{-1mm}

\begin{wraptable}{r}{0.32\textwidth}
    \vspace{-15pt}
    \caption{Prisoner's Dilemma}
    \label{table:ipd-payoff}
    \centering
    \begin{tabular}{c|cc}
        A1/A2 & C & D \\
        \hline
        C & (-1, -1) & (-3, 0) \\
        D & (0, -3) & (-2, -2)\\
    \end{tabular}
    \vspace{-10pt}
\end{wraptable}
LIO poses a challenge for theoretical analysis in general Markov games because each agent's policy and incentive function are updated using different trajectories but are coupled through the RL updates of all other agents.
Nonetheless, a complete analysis of exact LIO---using closed-form gradient ascent without policy gradient approximation---is tractable in repeated matrix games.
In the stateless Iterated Prisoner's Dilemma (IPD), for example, with payoff matrix in \Cref{table:ipd-payoff}, we prove in \Cref{app:analysis} 
the following:
\begin{restatable}{proposition}{propIPD}
\label{prop:lio-ipd}
Two LIO agents converge to mutual cooperation in the Iterated Prisoner's Dilemma.
\end{restatable}
\vspace{-5pt}
Moreover, we may gain further insight by visualizing the learning dynamics of exact LIO in the IPD, computed in \Cref{app:analysis}.
Let $\eta^1 \defeq [\eta^1_C, \eta^1_D] \in [0,3]^2$ be the incentives that Agent 1 gives to Agent 2 for cooperation (C) and defection (D), respectively.
Let $\theta^2$ denote Agent 2's probability of cooperation.
In \Cref{fig:vf-t2-e1}, the curvature of vector fields shows guaranteed increase in probability of recipient cooperation $\theta^2$ (vertical axis) along with increase in incentive value $\eta^1_C$ received for cooperation (horizontal axis).
For higher values of incentive for defection $\eta^1_D$, greater values of $\eta^1_C$ are needed for $\theta^2$ to increase.
\Cref{fig:app-vf-t2-e1} shows that incentive for defection is guaranteed to decrease.

\vspace{-2mm}
\section{Experimental setup}
\label{sec:experiment}
\vspace{-2mm}
Our experiments\footnote{Code for all experiments is available at \url{https://github.com/011235813/lio}} demonstrate that LIO agents are able to reach near-optimal individual performance by incentivizing other agents in cooperation problems with conflicting individual and group utilities.
We define three different environments with increasing complexity in \Cref{subsec:environments} and describe the implementation of our method and baselines in \Cref{subsec:implementation}.

\vspace{-1mm}
\subsection{Environments}
\label{subsec:environments}


\textbf{Iterated Prisoner's Dilemma (IPD).}
We test LIO on the memory-1 IPD as defined in \citep{foerster2018learning}, where agents observe the joint action taken in the previous round and receive extrinsic rewards in \Cref{table:ipd-payoff}.
This serves as a test of our theoretical prediction in \Cref{subsec:ipd-analysis}.

\textbf{$N$-Player Escape Room (ER).}
We experiment on the $N$-player Escape Room game shown in \Cref{fig:symmetric-room} (\Cref{sec:intro}).
By symmetry, any agent can receive positive extrinsic reward, as long as there are enough cooperators.
Hence, for methods that allow incentivization, every agent is both a reward giver and recipient.
We experiment with the cases $(N = 2, M=1)$ and $(N = 3, M = 2)$.
We also describe an asymmetric 2-player case and results in \Cref{app:asymmetric}.

\begin{wrapfigure}{r}{0.35\textwidth}
    \vspace{-20pt}
    \centering

    \includegraphics[width=1.0\linewidth]{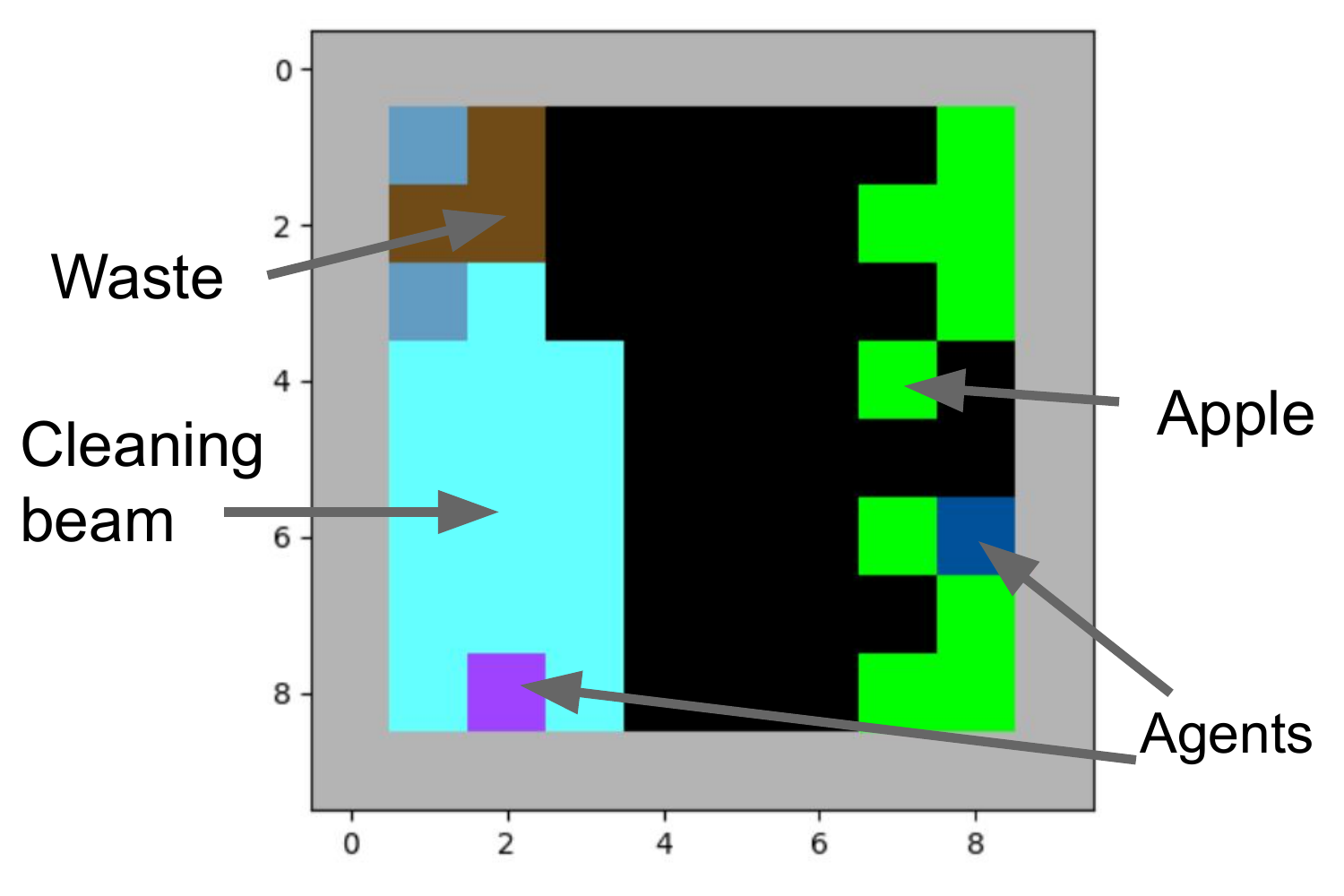}
    \caption{\textit{Cleanup} (10x10 map): apple spawn rate decreases with increasing waste, which agents can clear with a cleaning beam.}
    \label{fig:cleanup-10x10-map}
    \vspace{-25pt}
\end{wrapfigure}
\textbf{Cleanup.}
Furthermore, we conduct experiments on the Cleanup game (\Cref{fig:cleanup-10x10-map}) \citep{hughes2018inequity,wang2019evolving}.
Agents get +1 individual reward by collecting apples, which spawn on the right hand side of the map at a rate that decreases linearly to zero as the amount of waste in a river approaches a depletion threshold.
Each episode starts with a waste level above the threshold and no apples present.
While an agent can contribute to the public good by firing a cleaning beam to clear waste, it can only do so at the river as its fixed orientation points upward.
This would enable other agents to defect and selfishly collect apples, resulting in a difficult social dilemma.
Each agent has an egocentric RGB image observation that spans the entire map.

\vspace{-2mm}
\subsection{Implementation and baselines}
\label{subsec:implementation}
\vspace{-2mm}
We describe key details here and provide a complete description in \Cref{app:implementation}.
In each method, all agents have the same implementation without sharing parameters.
The incentive function of a LIO agent is a neural network defined as follows: its input is the concatenation of the agent's observation and all other agents' chosen actions;
the output layer has size $N$, sigmoid activation, and is scaled element-wise by a multiplier $R_{\text{max}}$;
each output node $j$, which is bounded in $[0, R_{\text{max}}]$, is interpreted as the real-valued reward given to agent with index $j$ in the game (we zero-out the value it gives to itself).
We chose $R_{\text{max}}=[3,2,2]$ for [IPD, ER, Cleanup], respectively, so that incentives can overcome any extrinsic penalty or opportunity cost for cooperation.
We use on-policy learning with policy gradient for each agent in IPD and ER, and actor-critic for Cleanup.
To ensure that all agents' policies perform sufficient exploration for the effect of incentives to be discovered, we include an exploration lower bound $\epsilon$ such that $\tilde{\pi}(a|s) = (1 - \epsilon) \pi(a|s) + \epsilon / |\Acal|$, with linearly decreasing $\epsilon$.

\textbf{Fully-decentralized implementation (LIO-dec).}
Each decentralized LIO agent $i$ learns a model of another agent's policy parameters $\theta^j$ via $\theta^j_{\text{estimate}} = \argmax_{\theta^j} \sum_{(o^j_t,a^j_t) \in \tau} \log \pi_{\theta^j}(a^j_t|o^j_t)$ at the end of each episode $\tau$.
With knowledge of agent $j$'s egocentric observation and individual rewards, it conducts incentive function updates using a fictitious policy update in \eqref{eq:recipient-update} with $\theta^j_{\text{estimate}}$ in place of $\theta^j$.

\textbf{Baselines.}
The first baseline is independent policy gradient, labeled \textbf{PG}, which has the same architecture as the policy part of LIO.
Second, we augment policy gradient with discrete ``give-reward'' actions, labeled \textbf{PG-d}, whose action space is $\Acal \times \lbrace \text{no-op}, \text{give-reward} \rbrace^{N-1}$.
We try reward values in the set $\lbrace 2, 1.5, 1.1 \rbrace$.
Giving reward incurs an equivalent cost.
Next, we design a more flexible policy gradient baseline called \textbf{PG-c}, which has continuous give-reward actions.
It has an augmented action space $\Acal \times [0,R_{\text{max}}]^{N-1}$ and learns a factorized policy $\pi(a_d,a_r|o) \defeq \pi(a_d|o) \pi(a_r|o)$, where $a_d \in \Acal$ is the regular discrete action and $a_r \in [0, R_{\text{max}}]^{N-1}$ is a vector of incentives given to the other $N-1$ agents.
\Cref{app:implementation} describes how PG-c is trained.
In ER, we run \textbf{LOLA-d} and \textbf{LOLA-c} with the same augmentation scheme as PG-d and PG-c.
In Cleanup, we compare with independent actor-critic agents (\textbf{AC-d} and \textbf{AC-c}), which are analogously augmented with ``give-reward'' actions,
and with inequity aversion (\textbf{IA}) agents \citep{hughes2018inequity}.
We also show the approximate upper bound on performance by training a fully-centralized actor-critic (\textbf{Cen}) that is (unfairly) allowed to optimize joint reward.

\section{Results}
\label{sec:results}
\vspace{-1mm}
We find that LIO agents reach near-optimal \textit{collective} performance in all three environments, despite being designed to optimize only \textit{individual} rewards.
This arose in ER and Cleanup because incentivization enabled agents to find an optimal division of labor\footnote{Learned behavior in Cleanup can be viewed at \url{https://sites.google.com/view/neurips2020-lio}}
and in IPD where LIO is proven to converge to the CC solution.
In contrast, various baselines displayed competitive behavior that led to suboptimal solutions, were not robust across random seeds, or failed to cooperate altogether.
We report the results of 20 independent runs for IPD and ER, and 5 runs for Cleanup.

\begin{wrapfigure}{r}{0.25\textwidth}
    \vspace{-12pt}
  \centering
    \includegraphics[width=0.25\textwidth]{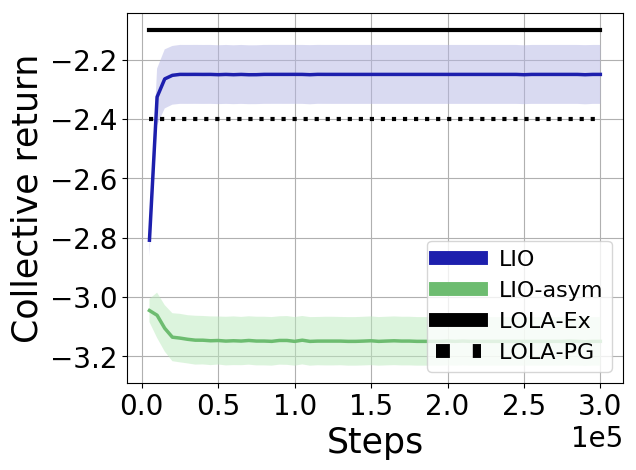}
  \caption{The sum of all agents' rewards in IPD.}
  \label{fig:ipd}
  \vspace{-10pt}
\end{wrapfigure}
\textbf{Iterated Prisoner's Dilemma.}
In accord with the theoretical prediction of exact LIO in \Cref{subsec:ipd-analysis} and \Cref{app:analysis}, two LIO agents with policy gradient approximation converge near the optimal CC solution with joint reward -2 in the IPD (\Cref{fig:ipd}).
This meets the performance of LOLA-PG and is close to LOLA-Ex, as reported in \citep{foerster2018learning}.
In the asymmetric case (LIO-asym) where one LIO agent is paired with a PG agent, we indeed find that they converge to the DC solution: PG is incentivized to cooperate while LIO defects, resulting in collective reward near -3.


\begin{figure*}[t]
\centering
\begin{subfigure}[t]{0.25\linewidth}
    \centering
    \includegraphics[width=1.0\linewidth]{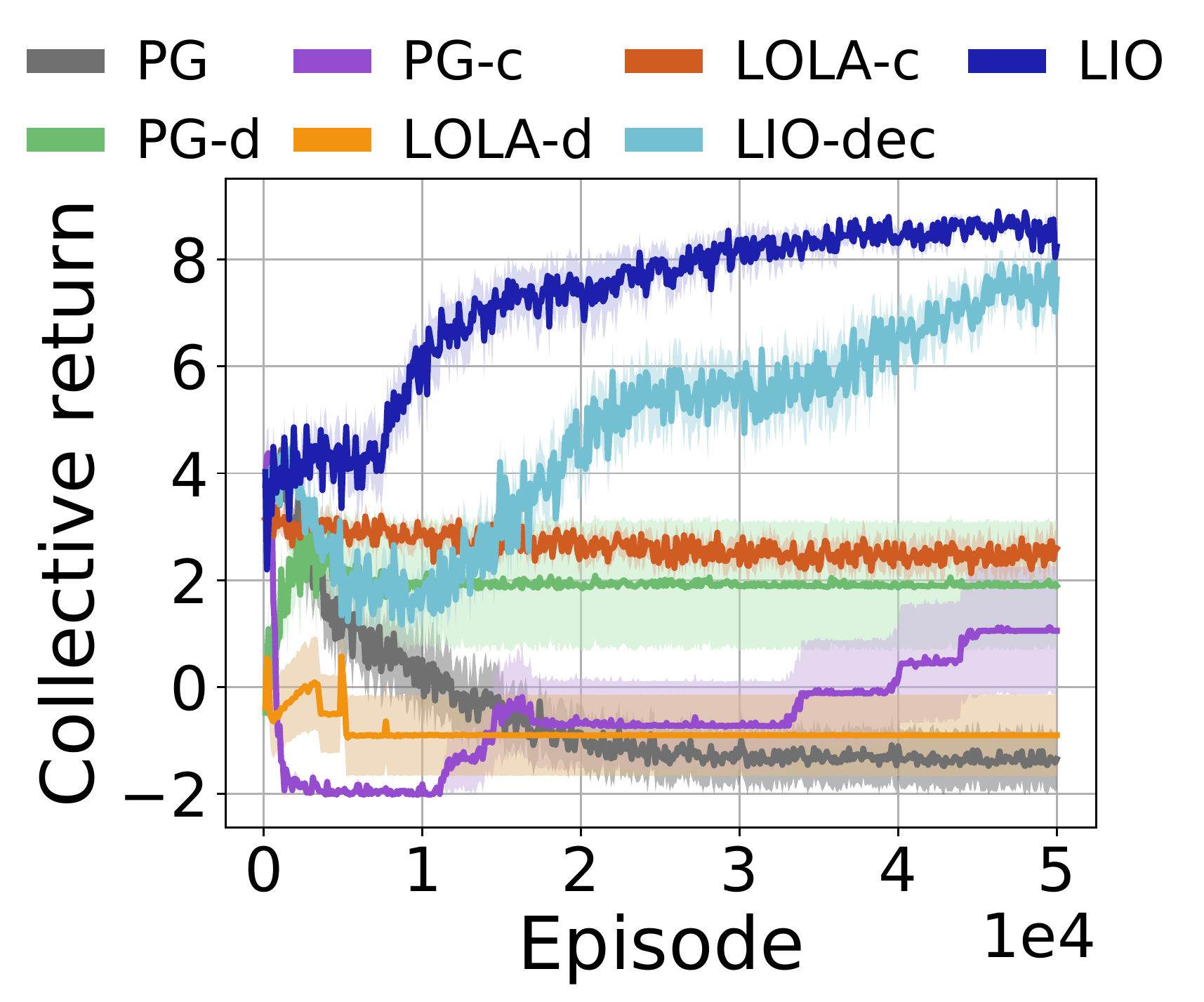}
    \caption{Collective return in $N$=2}
    \label{fig:er_n2}
\end{subfigure}
\hfill
\begin{subfigure}[t]{0.24\linewidth}
    \centering
    \includegraphics[width=1.0\linewidth]{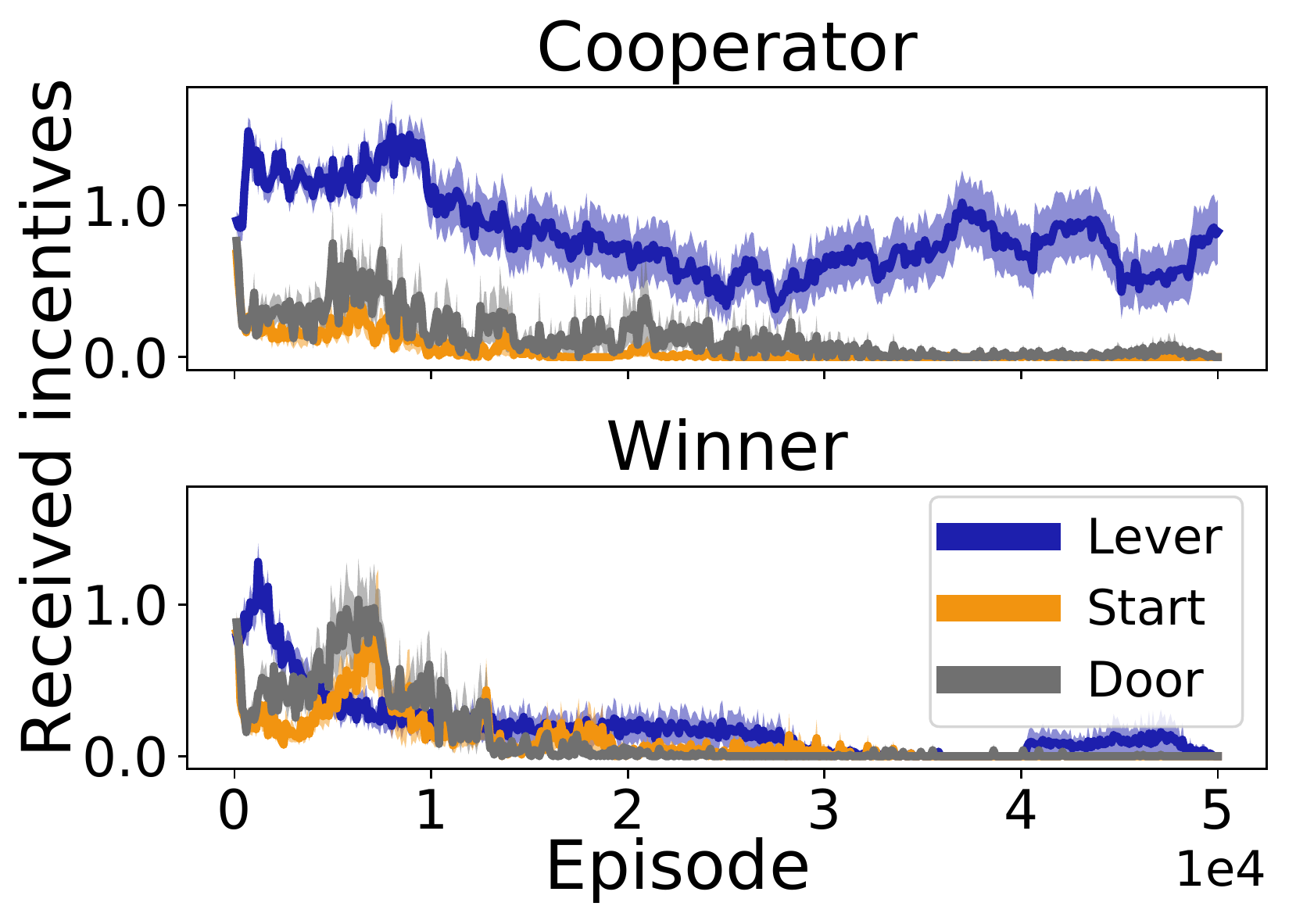}
    \caption{Incentives in $N$=2}
    \label{fig:er_n2_received}
\end{subfigure}
\hfill
\begin{subfigure}[t]{0.25\linewidth}
    \centering
    \includegraphics[width=1.0\linewidth]{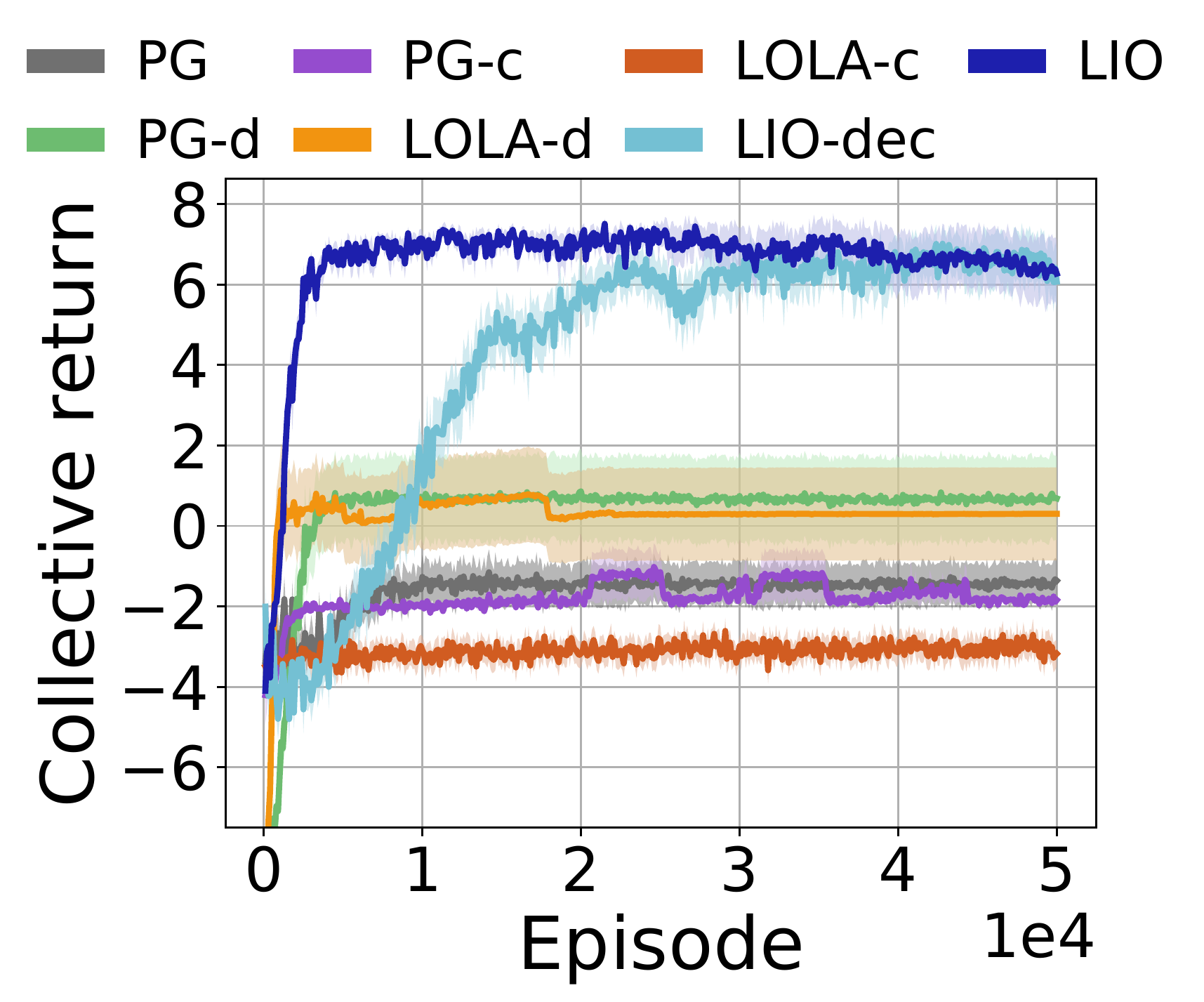}
    \caption{Collective return in $N$=3}
    \label{fig:er_n3}
\end{subfigure}
\hfill
\begin{subfigure}[t]{0.24\linewidth}
    \centering
    \includegraphics[width=1.0\linewidth]{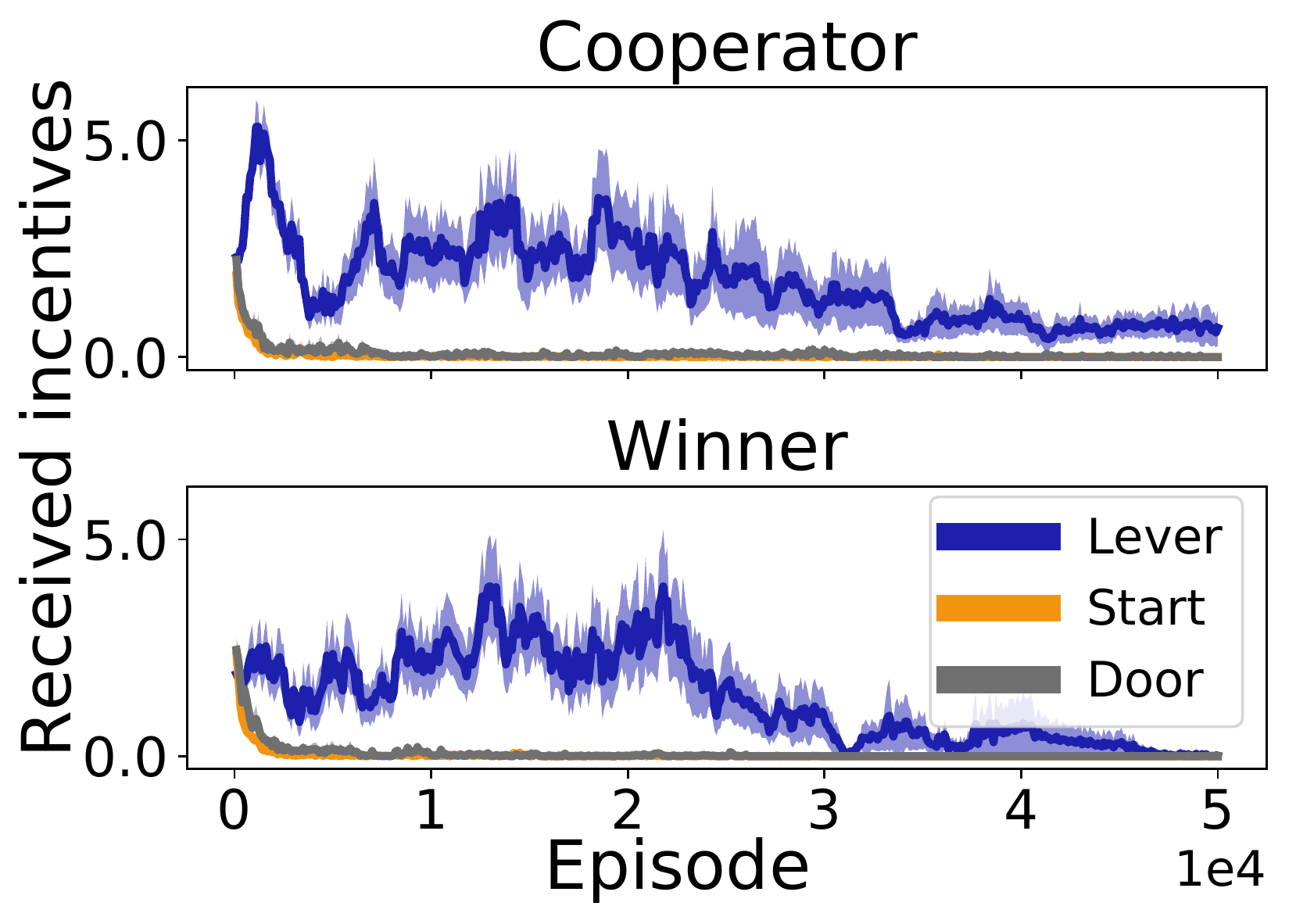}
    \caption{Incentives in $N$=3}
    \label{fig:er_n3_received}
\end{subfigure}
\vspace{-1mm}
\caption{Escape Room. (a,c) LIO agents converge near the global optimum with value 9 (N=2) and 8 (N=3). (b,d) Incentives received for each action by the agent who ends up going to the lever/door.}
\label{fig:results_er}
\vspace{-2mm}
\end{figure*}


\textbf{Escape Room.}
\Cref{fig:er_n2,fig:er_n3} show that groups of LIO agents discover a division of labor in both ER(2,1) and ER(3,2), whereby some agent(s) cooperate by pulling the lever to allow another agent to exit the door, such that collective return approaches the optimal value (9 for the 2-player case, 8 for the 3-player case).
Fully-decentralized LIO-dec successfully solved both cases, albeit with slower learning speed.
As expected, PG agents were unable to find a cooperative solution: they either stay at the start state or greedily move to the door, resulting in negative collective return.
The augmented baselines PG-d and PG-c sometimes successfully influence the learning of another agent to solve the game, but exhibit high variance across independent runs.
This is strong evidence that conventional RL alone is not well suited for learning to incentivize, as the effect of ``give-reward'' actions manifests only in future episodes.
LOLA succeeds sometimes but with high variance, as it does not benefit from the stabilizing effects of online cross-validation and separation of the incentivization channel from regular actions.
\Cref{app:asymmetric} contains results in an asymmetric case (LIO paired with PG), where we compare to an additional heuristic two-timescale baseline and a variant of LIO.
\Cref{fig:er_n5} evidences that LIO scales well to larger groups such as ER(5,3), since the complexity of \eqref{eq:chain-rule} is linear in number of agents.

To understand the behavior of LIO's incentive function, we classify each agent \textit{at the end of training} as a ``Cooperator'' or ``Winner'' based on whether its final policy has greater probability of going to the lever or door, respectively.
For each agent type, aggregating over all agents of that type, we measure incentives received by that agent type when it takes each of the three actions during training.
\Cref{fig:er_n2_received,fig:er_n3_received} show that the Cooperator was correctly incentivized for pulling the lever and receives negligible incentives for noncooperative actions.
Asymptotically, the Winner receives negligible incentives from the Cooperator(s), who learned to avoid the cost for incentivization \eqref{eq:reward-regularizer} when doing so has no benefits itself, whereas incentives are still nonzero for the Cooperator.

\begin{figure*}[t]
\centering
\begin{subfigure}[t]{0.24\linewidth}
    \centering
    \includegraphics[width=1.0\linewidth]{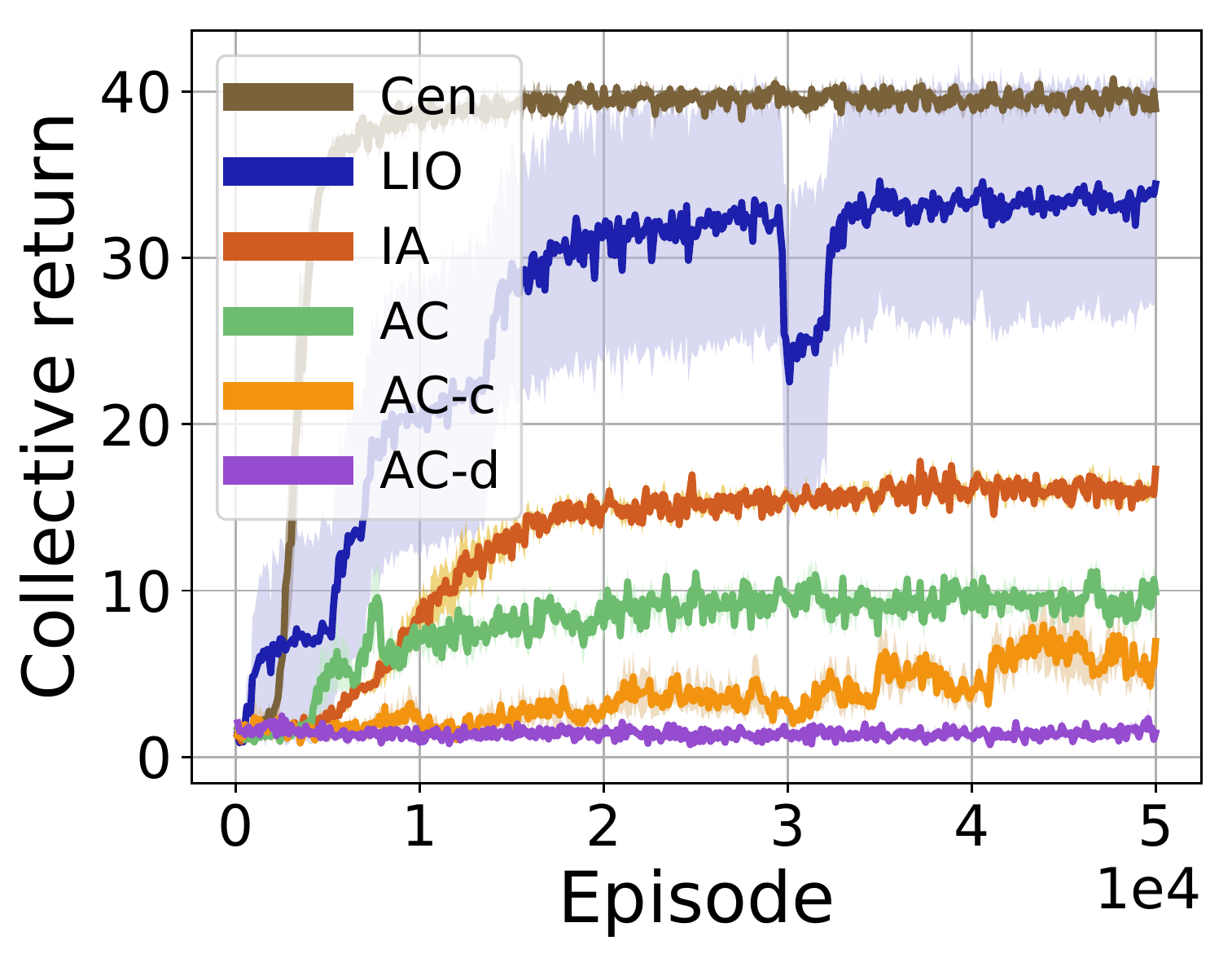}
    \caption{Cleanup 7x7}
    \label{fig:cleanup-small-reward}
\end{subfigure}
\hfill
\begin{subfigure}[t]{0.25\linewidth}
    \centering
    \includegraphics[width=1.0\linewidth]{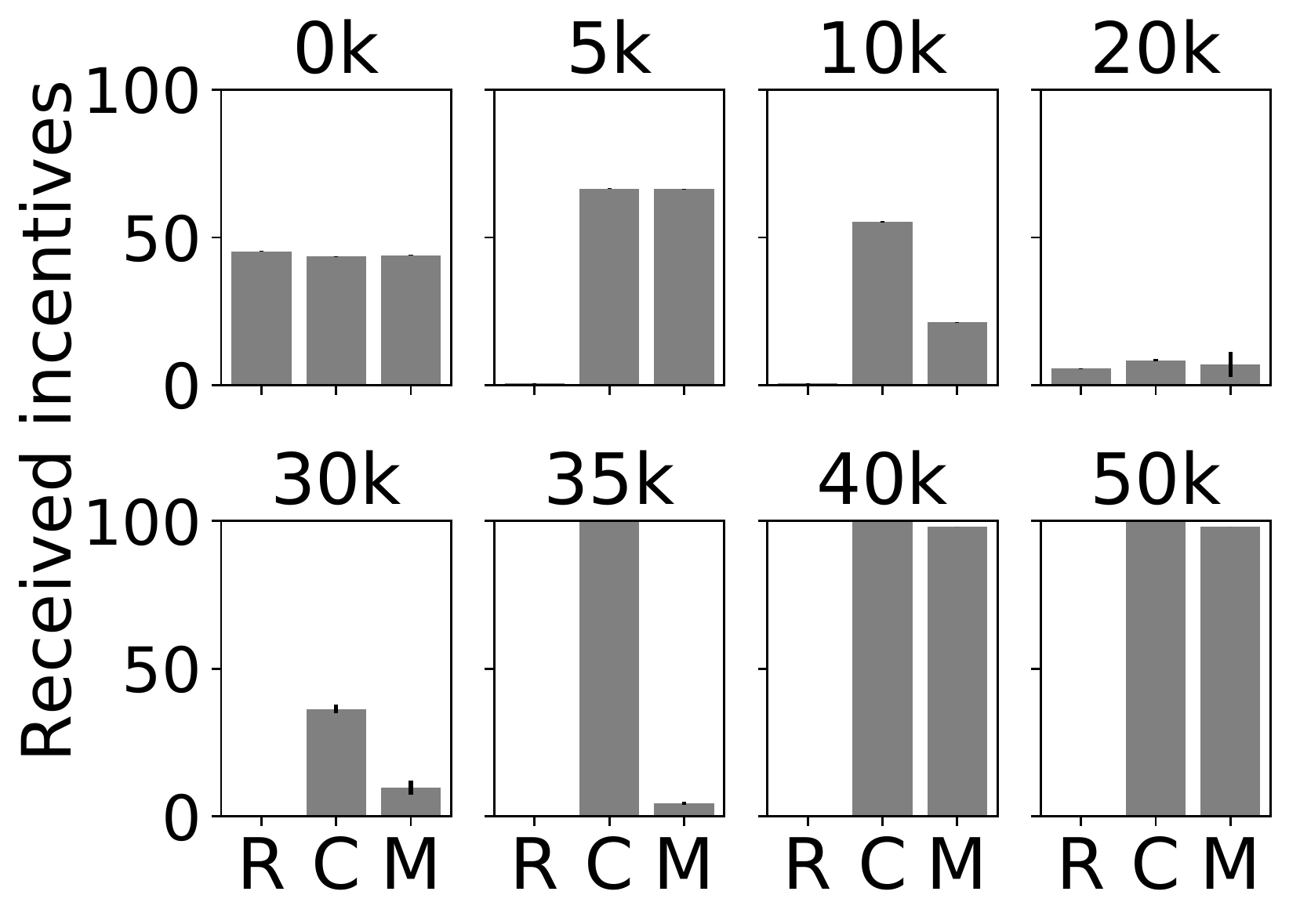}
    \caption{Scripted agent incentives}
    \label{fig:cleanup-small-measure}
\end{subfigure}
\hfill
\begin{subfigure}[t]{0.24\linewidth}
    \centering
    \includegraphics[width=1.0\linewidth]{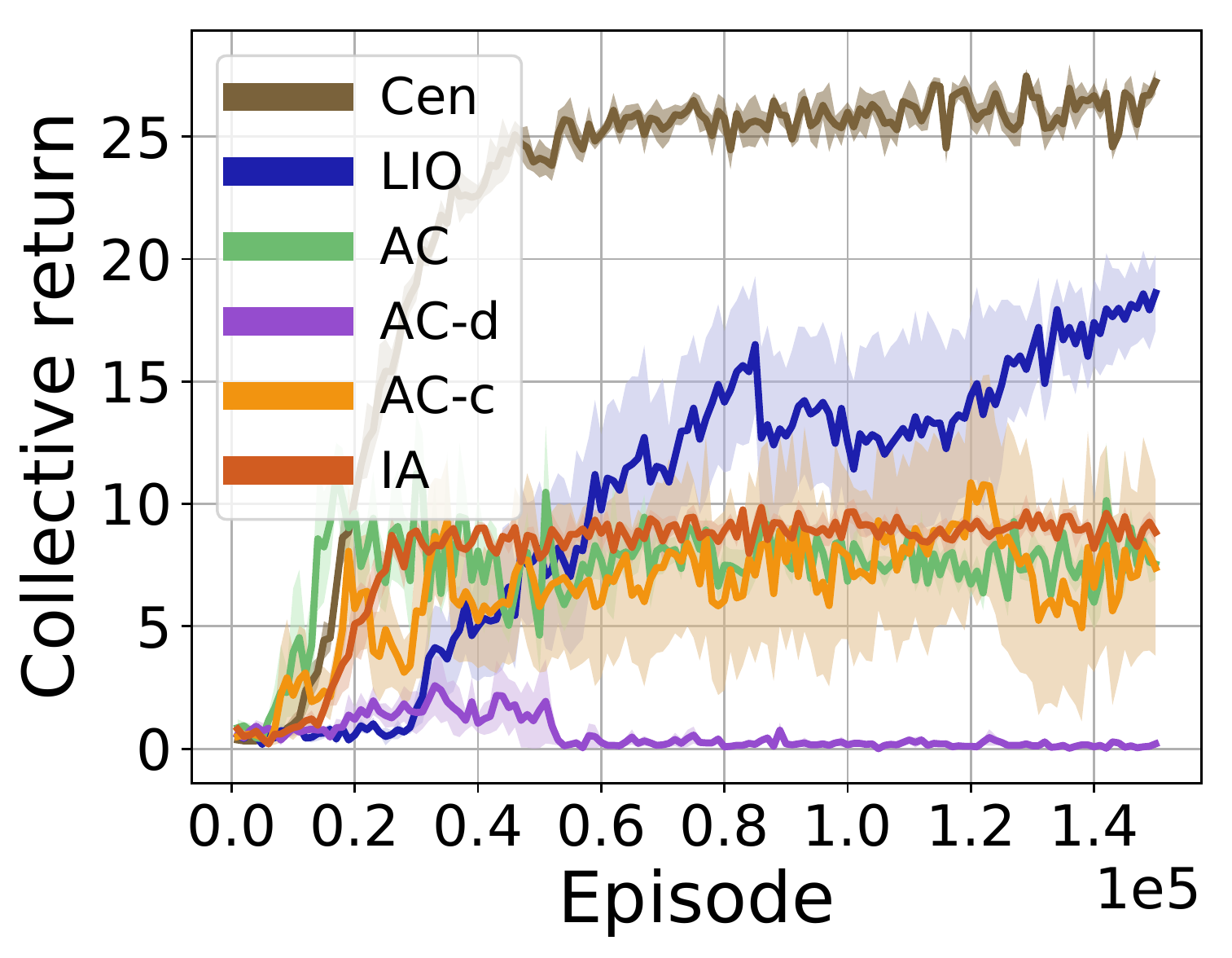}
    \caption{Cleanup 10x10}
    \label{fig:cleanup-10x10-reward}
\end{subfigure}
\hfill
\begin{subfigure}[t]{0.25\linewidth}
    \centering
    \includegraphics[width=1.0\linewidth]{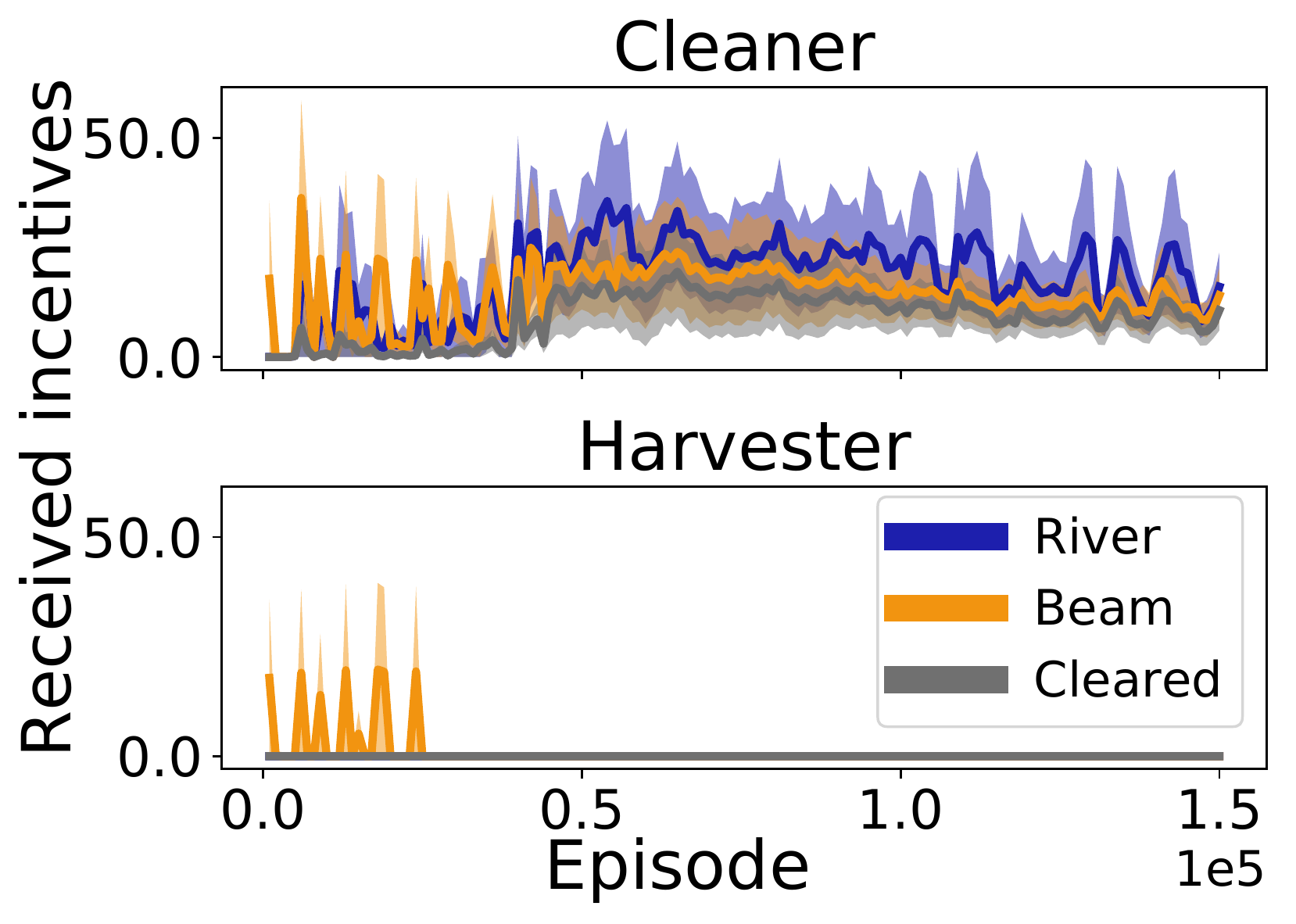}
    \caption{Incentives in training}
    \label{fig:cleanup-10x10-received}
\end{subfigure}
\vspace{-1mm}
\caption{Results on Cleanup. (a,c) Emergent division of labor between LIO agents enables higher performance than AC and IA baselines, which find rewards but exhibit competitive behavior.
(b) Behavior of incentive function in 7x7 Cleanup at different training checkpoints, measured against three scripted opponents: R moves within river without cleaning; C successfully cleans waste; M fires the cleaning beam but misses waste (mean and standard error of 20 evaluation episodes). 
(d) 10x10 map: the LIO agent who becomes a ``Cleaner'' receives incentives, while the ``Harvester'' does not.}
\label{fig:cleanup}
\vspace{-5mm}
\end{figure*}



\textbf{Cleanup.}
\Cref{fig:cleanup-small-reward,fig:cleanup-10x10-reward} show that LIO agents collected significantly more extrinsic rewards than AC and IA baselines in Cleanup,
and approach the upper bound on performance as indicated by Cen,
on both a 7x7 map and a 10x10 map with more challenging depletion threshold and lower apple respawn rates.
LIO agents discovered a division of labor (\Cref{fig:cleanup-small-lio-behavior}), whereby one agent specializes to cleaning waste at the river while the other agent, who collects almost all of the apples, provides incentives to the former.
In contrast, AC baselines learned clean but subsequently compete to collect apples, which is suboptimal for the group (\Cref{fig:cleanup-small-ac-behavior}).
Due to continual exploration by all agents, an agent may change its behavior if it receives incentives for ``wrong actions'': e.g., near episode 30k in \Cref{fig:cleanup-small-reward}, an agent temporarily stopped cleaning the river despite having consistently done so earlier.

We can further understand the progression of LIO's incentive function during training as follows.
First, we classify LIO agents \textit{at the end of training} as a ``Cleaner'' or a ``Harvester'', based on whether it primarily cleans waste or collects apples, respectively.
Next, we define three hand-scripted agents:
an R agent moves in the river but does not clean, a C agent successfully cleans waste, and an M agent fires the cleaning beam but misses waste.
\Cref{fig:cleanup-small-measure} shows the incentives given by a Harvester to these scripted agents when they are tested together periodically during training.
At episodes 10k, 30k and 35k, it gave significantly more incentives to C than to M, meaning that it distinguished between successful and unsuccessful cleaning, which explains how its actual partner in training was incentivized to become a Cleaner.
After 40k episodes, it gives nonzero reward for ``fire cleaning beam but miss'', likely because its actual training partner already converged to successful cleaning (\Cref{fig:cleanup-small-reward}), so it may have “forgotten” the difference between successful and unsuccessful usage of the cleaning beam.
As shown by results in the Escape Room (\Cref{fig:er_n2_received,fig:er_n3_received}), correct incentivization can be maintained if agents have a sufficiently large lower bound on exploration rates that pose the risk of deviating from cooperative behavior.
\Cref{fig:cleanup-10x10-received} shows the actual incentives received by Cleaner and Harvester agents when they are positioned in the river, fire the cleaning beam, or successfully clear waste during training.
We see that asymptotically, only Harvesters provide incentives to Cleaners and not the other way around.

\section{Conclusion and future directions}
\label{sec:conclusion}
We created Learning to Incentivize Others (LIO), an agent who learns to give rewards directly to other RL agents.
LIO learns an incentive function by explicitly accounting for the impact of incentives on its own extrinsic objective, through the learning updates of reward recipients.
In the Iterated Prisoner's Dilemma, an illustrative \textit{Escape Room} game, and a benchmark social dilemma problem called \textit{Cleanup}, LIO correctly incentivizes other agents to overcome extrinsic penalties so as to discover cooperative behaviors, such as division of labor, and achieve near-optimum collective performance.
We further demonstrated the feasibility of a fully-decentralized implementation of LIO.
%

Our approach to the goal of ensuring cooperation in a decentralized multi-agent population poses many open questions.
1) How should one analyze the simultaneous processes of learning incentive functions, which continuously modifies the set of equilibria, and learning policies with these changing rewards? 
While previous work have treated the convergence of gradient-based learning in differentiable games with fixed rewards \citep{balduzzi2018mechanics,letcher2019stable}, the theoretical analysis of learning processes that dynamically change the reward structure of a game deserves more attention.
2) How can an agent account for the cost of incentives in an adaptive way? 
An improvement to LIO would be a handcrafted or learned mechanism that prevents the cost from driving the incentive function to zero before the effect of incentives on other agents' learning is measurable.
3) How should agents better account for the longer-term effect of incentives?
One possibility is to differentiate through a sequence of gradient descent updates by recipients, during which the incentive function is fixed.
4) Should social factors modulate the effect of incentives in an agent population?
LIO assumes that recipients cannot reject an incentive, but a more intelligent agent may selectively accept a subset of incentives based on its appraisal of the other agents' behavior.
We hope our work sparks further interest in this research endeavor.




\section*{Broader Impact}

Our work is a step toward the goal of ensuring the common good in a potential future where independent reinforcement learning agents interact with one another and/or with humans in the real world.
We have shown that cooperation can emerge by introducing an additional learned incentive function that enables one agent to affect another agent's reward directly.
However, as agents still independently maximize their own individual rewards, it is open as to how to prevent an agent from misusing the incentive function to exploit others.
One approach for future research to address this concern is to establish new connections between our work and the emerging literature on reward tampering \citep{everitt2019reward}.
By sparking a discussion on this important aspect of multi-agent interaction, we believe our work has a positive impact on the long-term research endeavor that is necessary for RL agents to be deployed safely in real-world applications.



\section*{Acknowledgements}

We thank Thomas Anthony, Jan Balaguer, and Thore Graepel at DeepMind for insightful discussions and feedback.
JY was funded in part by NSF III-1717916.







\bibliographystyle{natbib}
\bibliography{citation}

\newpage
\appendix

\section{Further discussion}

\subsection{Cost for incentivization}
\label{app:cost}

We justify the way in which LIO accounts for the cost of incentivization as follows.
Recall that this cost is incurred in the objective for LIO's incentive function (see \eqref{eq:reward-objective} and \eqref{eq:reward-regularizer}), instead of being accounted in the total reward \eqref{eq:recipient-reward} that is maximized by LIO's policy.
Fundamentally, the reason is that the cost should be incurred only by the part of the agent that is directly responsible for incentivization.
In LIO, the policy and incentive function are separate modules: while the former takes regular actions to maximize \textit{external} rewards, only the latter produces incentives that directly and actively shape the behavior of other agents.
The policy is decoupled from incentivization, and it would be incorrect to penalize it for the behavior of the incentive function.
Instead, we need to attribute the cost directly to the incentive function parameters via \eqref{eq:reward-regularizer}.
From a more intuitive perspective, LIO is constructed with the knowledge that it can perform two fundamentally different behaviors---1) take regular actions that affect the Markov game transition, and 2) give incentives to shape other agents' learning---and it knows not to penalize the former behavior with the latter behavior.
In contrast, if one were to augment conventional RL with reward-giving actions (as we do for baselines in \Cref{subsec:implementation}), then the cost for incentivization should indeed be accounted by the policy.
One may consider other mechanisms for cost, such as budget constraints \citep{lupu2020gifting}.

In our experiments, we find the coefficient $\alpha$ in the cost for incentivization is a sensitive parameter.
At the beginning of training, \eqref{eq:reward-regularizer} immediately drives the magnitude of incentives to zero.
However, both the reward-giver and recipients require sufficient time to learn the effect of incentives, which means that too large an $\alpha$ would lead to the degenerate result of $r_{\eta^i} = \mathbf{0}$.
On the other extreme, $\alpha=0$ means there is no penalty and may result in profligate incentivization that serves no useful purpose.
While we found that values of $10^{-3}$ and $10^{-4}$ worked well in our experiments, one may consider adaptive and dynamic computation of $\alpha$ for more efficient training.

\section{Analysis in Iterated Prisoner's Dilemma}
\label{app:analysis}

\propIPD*

\begin{proof}
We prove this by deriving closed-form expressions for the updates to parameters of policies and incentive functions.
These updates are also used to compute the vector fields shown in \Cref{fig:vf-t2-e1}.
Let $\theta^i$ for $i \in \lbrace 1, 2 \rbrace$ denote each agent's probability of taking the cooperative action.
Let $\eta^1 \defeq [\eta^1_C, \eta^1_D] \in \Rbb^2$ denote Agent 1's incentive function, where the values are given to Agent 2 when it takes action $a^2 = C$ or $a^2 = D$.
Similarly, let $\eta^2$ denote Agent 2's incentive function.
The value function for each agent is defined by
\begin{equation}
\begin{aligned}
    V^i(\theta^1, \theta^2) &= \sum_{t=0}^{\infty} \gamma^t p^T r^i = \frac{1}{1 - \gamma}p^T r^i \, ,\\
    \text{where } \quad p &= \left[ \theta^1 \theta^2, \theta^1 (1 - \theta^2), (1-\theta^1)\theta^2, (1-\theta^1)(1-\theta^2) \right] \, .
\end{aligned}    
\end{equation}
The total reward received by each agent is
\begin{align}
    r^1 &= \left[ -1 + \eta^2_C, -3 + \eta^2_C, 0 + \eta^2_D, -2 + \eta^2_D \right] \, ,\\
    r^2 &= \left[ -1 + \eta^1_C, 0 + \eta^1_D, -3 + \eta^1_C, -2 + \eta^1_D \right] \, .
\end{align}
Agent 2 updates its policy via the update
\begin{equation}
\begin{aligned}\label{eq:ipd-theta2-update}
    \hat{\theta}^2 &= \theta^2 + \alpha \nabla_{\theta^2} V^2(\theta^1, \theta^2) \\
    &= \theta^2 + \frac{\alpha}{1-\gamma} \nabla_{\theta^2} \left( \theta^1\theta^2(-1 + \eta^1_C) + \theta^1(1-\theta^2)\eta^1_D \right. \\
    &\quad + \left. (1-\theta^1)\theta^2 (-3+\eta^1_C) + (1-\theta^1)(1-\theta^2)(-2+\eta^1_D) \right) \\
    &= \theta^2 + \frac{\alpha}{1-\gamma} \left( \eta^1_C - \eta^1_D - 1 \right) \, ,
\end{aligned}    
\end{equation}
and likewise for Agent 1:
\begin{align}\label{eq:ipd-theta1-update}
    \hat{\theta}^1 &= \theta^1 + \frac{\alpha}{1-\gamma} \left( \eta^2_C - \eta^2_D - 1 \right) \, .
\end{align}
Let $\hat{p}$ denote the joint action probability under updated policies $\hat{\theta}^1$ and $\hat{\theta}^2$, and let $\Delta^2 \defeq (\eta^1_C - \eta^1_D - 1)\alpha / (1-\gamma)$ denote Agent 2's policy update. 
Agent 1 updates its incentive function parameters via
\begin{equation}
\begin{aligned}
    \eta^1 &\leftarrow \eta^1 + \beta \nabla_{\eta^1} \frac{1}{1 - \gamma} \hat{p}^T r^1 \\
    &= \eta^1 + \frac{\beta}{1-\gamma} \nabla_{\eta^1} \left[ \hatt^1(\theta^2+\Delta^2)(-1+\eta^2_C) + \hatt^1(1 - \theta^2-\Delta^2)(-3 + \eta^2_C) \right. \\
    &\quad \left. + (1-\hatt^1)(\theta^2+\Delta^2)\eta^2_D + (1-\hatt^1)(1-\theta^2-\Delta^2)(-2 + \eta^2_D) \right] \\
    &= \eta^1 + \frac{\beta \alpha}{(1-\gamma)^2} B_2 \begin{bmatrix}1 \\ -1\end{bmatrix} \, , \label{eq:ipd-eta1-update}
\end{aligned}
\end{equation}
where the scalar $B_2$ is
\begin{align}
    B_2 &= \hat{\theta}^1 (-1 + \eta^2_C) - \hat{\theta}^1 (-3 + \eta^2_C) + (1 - \hat{\theta}^1) \eta^2_D - (1 - \hat{\theta}^1)(-2 + \eta^2_D) = 2 \, .
\end{align}
By symmetry, with $B_1 = 2$, Agent 2 updates its incentive function via
\begin{align}\label{eq:ipd-eta2-update}
    \eta^2 \leftarrow \eta^2 + \frac{\beta\alpha}{(1-\gamma)^2} B_1 \begin{bmatrix}1 \\ -1\end{bmatrix} \, .
\end{align}
Note that each $\eta^i$ is updated so that $\eta^i_C$ increases while $\eta^i_D$ decreases.
Referring to \eqref{eq:ipd-theta2-update} and \eqref{eq:ipd-theta1-update}, one sees that the updates to incentive parameters lead to updates to policy parameters that increase the probability of mutual cooperation.
This is consistent with the viewpoint of modifying the Nash Equilibrium of the payoff matrices.
With incentives, the players have payoff matrices in \Cref{tab:ipd-modified-payoff}.
For CC to be the global Nash Equilibrium, such that cooperation is preferred by an agent $i$ regardless of the other agent's action, incentives must satisfy $\eta^i_C - \eta^i_D - 1 > 0$.
This is guaranteed to occur by incentive updates \eqref{eq:ipd-eta1-update} and \eqref{eq:ipd-eta2-update}.
\end{proof}

\begin{table}[ht]
    \centering
    \caption{Payoff matrices for row player (left) and column player (right) with incentives.}
    \begin{tabular}{c|cc}
        A1 & C & D \\
         \hline
        C & -1 + $\eta^2_C$ & -3 + $\eta^2_C$ \\
        D & 0 + $\eta^2_D$ & -2 + $\eta^2_D$
    \end{tabular}
    \hfil
    \begin{tabular}{c|cc}
        A2 & C & D \\
         \hline
        C & -1 + $\eta^1_C$ & 0 + $\eta^1_D$ \\
        D & -3 + $\eta^1_C$ & -2 + $\eta^1_D$
    \end{tabular}
    \label{tab:ipd-modified-payoff}
\end{table}

\begin{figure}[t]
\centering
\begin{subfigure}[t]{0.24\linewidth}
    \centering
    \includegraphics[width=1.0\linewidth]{vectorfields/vf-theta2-eta1c-eta1d0p0.png}
\end{subfigure}
\hfill
\begin{subfigure}[t]{0.24\linewidth}
    \centering
    \includegraphics[width=1.0\linewidth]{vectorfields/vf-theta2-eta1c-eta1d1p0.png}
\end{subfigure}
\hfill
\begin{subfigure}[t]{0.24\linewidth}
    \centering
    \includegraphics[width=1.0\linewidth]{vectorfields/vf-theta2-eta1c-eta1d2p0.png}
\end{subfigure}
\hfill
\begin{subfigure}[t]{0.24\linewidth}
    \centering
    \includegraphics[width=1.0\linewidth]{vectorfields/vf-theta2-eta1c-eta1d3p0.png}
\end{subfigure}

\begin{subfigure}[t]{0.24\linewidth}
    \centering
    \includegraphics[width=1.0\linewidth]{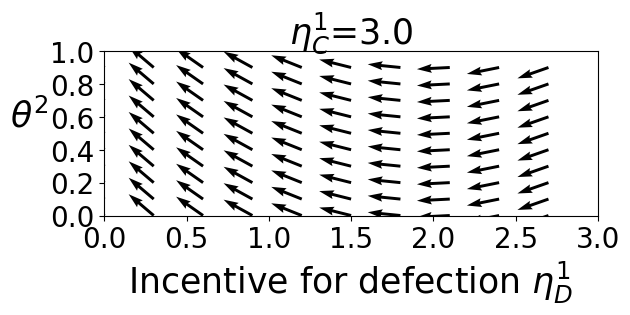}
    \label{fig:vf-t2-e1d-1}
\end{subfigure}
\hfill
\begin{subfigure}[t]{0.24\linewidth}
    \centering
    \includegraphics[width=1.0\linewidth]{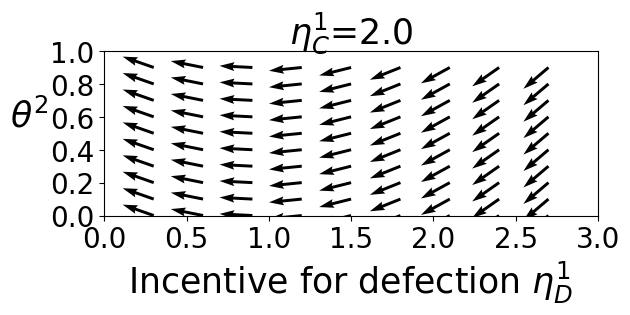}
    \label{fig:vf-t2-e1d-2}
\end{subfigure}
\hfill
\begin{subfigure}[t]{0.24\linewidth}
    \centering
    \includegraphics[width=1.0\linewidth]{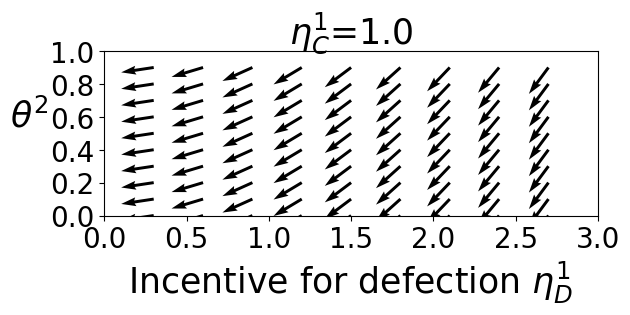}
    \label{fig:vf-t2-e1d-3}
\end{subfigure}
\hfill
\begin{subfigure}[t]{0.24\linewidth}
    \centering
    \includegraphics[width=1.0\linewidth]{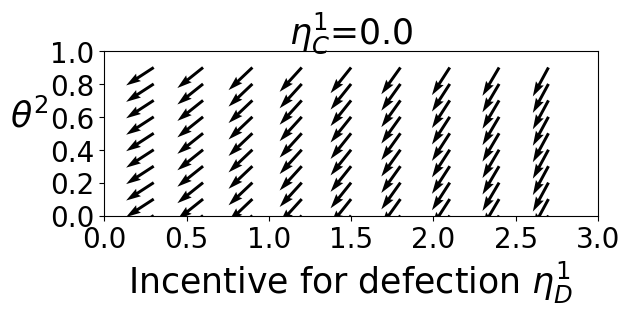}
    \label{fig:vf-t2-e1d-4}
\end{subfigure}
\caption{Vector fields showing the probability of recipient cooperation versus incentive value given for cooperation (top row) and defection (lower row). Each plot has a fixed value for the incentive given for the other action.}
\label{fig:app-vf-t2-e1}
\end{figure}

\section{Derivations}
\label{app:derivation}

The factor $\nabla_{\hat{\theta}^j} J^i(\hat{\tau}^i, \hat{\thetabf})$ \eqref{eq:chain-rule-factor} in the incentive function's gradient \eqref{eq:chain-rule} is derived as follows.
For brevity, we will drop the ``hat'' notation---recall that it indicates a quantity belongs to a new trajectory after a regular policy update---as all quantities here have ``hats''.
Let $\nabla_j$ denote $\nabla_{\hat{\theta}^j}$ and $\pi$ denote $\pi(a_t|s_t)$.
Let $V^{i,\pibf}(s)$ and $Q^{i,\pibf}(s,\abf)$ denote the global value and action-value function for agent $i$'s reward under joint policy $\pibf$. 
Then the gradient of agent $i$'s expected extrinsic return with respect to agent $j$'s policy parameter can be derived in a similar manner as standard policy gradients \citep{sutton2000policy}:
\begin{align*}
    &\nabla_j J^i(\tau,\thetabf) = \nabla_j V^{i,\pibf}(s_0) = \nabla_j \sum_{\abf} \pibf(\abf|s_0) Q^{i,\pibf}(s_0,\abf) \\
    &= \sum_{\abf} \pi^{-j} \left( (\nabla_j \pi^j) Q^{i,\pibf}(s_0,\abf) + \pi^j \nabla_j Q^{i,\pibf}(s_0,\abf) \right)\\
    &= \sum_{\abf} \pi^{-j} \left( (\nabla_j \pi^j)Q^{i,\pibf} + \pi^j \nabla_j \left( r^i + \gamma \sum_{s'} P(s'|s_0,\abf) V^{i,\pibf}(s') \right) \right) \\
    &= \sum_{\abf} \pi^{-j} \left( (\nabla_j \pi^j)Q^{i,\pibf} + \gamma \pi^j \sum_{s'} P(s'|s_0,\abf) \nabla_j V^{i, \pibf}(s') \right) \\
    &= \sum_{x} \sum_{k=0}^{\infty} P(s_0\rightarrow x,k, \pibf) \gamma^k \sum_{\abf} \pi^{-j} \nabla_j \pi^j Q^{i,\pibf}(x,\abf) \\
    &= \sum_s d^{\pibf}(s) \sum_{\abf} \pi^{-j} \nabla_j \pi^j Q^{i,\pibf}(s,\abf) \\
    &= \sum_s d^{\pibf}(s) \sum_{\abf} \pi^{-j} \pi^j \nabla_j \log \pi^j Q^{i,\pibf}(s,\abf) \\
    &= \Ebb_{\pibf} \left[ \nabla_j \log \pi^j(a^j|s) Q^{i,\pibf}(s,\abf) \right]
\end{align*}

Alternatively, one may rely on automatic differentiation in modern machine learning frameworks \citep{abadi2016tensorflow} to compute the chain rule \eqref{eq:chain-rule} via direct minimization of the loss \eqref{eq:eta-loss}.
This is derived as follows.
Let the notation $\neq j, i$ denote all indices except $j$ and $i$.
Note that agent $i$'s updated policy $\hat{\pi}^i$ is not a function of $\eta^i$, as it does not receive incentives from itself.
Recall that a recipient $j$'s updated policy $\hat{\pi}^j$ has explicit dependence on a reward-giver $i$'s incentive parameters $\eta^i$.
Also note that 
\begin{align*}
    \nabla_{\eta^i} \hat{\pi}^{-i} = \sum_{j\neq i} (\nabla_{\eta^i} \hat{\pi}^j) \hat{\pi}^{\neq j,i}
\end{align*}
by the product rule.
Then we have:
\begin{align*}
    & \nabla_{\eta^i} J^i(\hat{\tau}^i, \hat{\thetabf}) = \nabla_{\eta^i} V^{i,\hat{\pibf}}(\hat{s}_0) = \nabla_{\eta^i} \sum_{\hat{\abf}} \hat{\pi}^i(\hat{a}^i|\hat{s}_0)\hat{\pi}^{-i}(\hat{a}^{-i}|\hat{s}_0) Q^{i, \hat{\pibf}}(\hat{s}_0,\hat{\abf}) \\
    &= \sum_{\hat{\abf}} \hat{\pi}^i \left( \sum_{j\neq i} (\nabla_{\eta^i} \hat{\pi}^j) \hat{\pi}^{\neq j, i} Q^{i, \hat{\pibf}} + \hat{\pi}^{-i} \nabla_{\eta^i} Q^{i, \hat{\pibf}} \right) \quad \text{(by the remarks above)} \\
    &= \sum_{\hat{\abf}} \hat{\pi}^i \left( \sum_{j\neq i} (\nabla_{\eta^i} \hat{\pi}^j) \hat{\pi}^{\neq j, i} Q^{i, \hat{\pibf}} + \gamma \hat{\pi}^{-i} \sum_{s'} P(s'|\hat{s}_0,\hat{\abf}) \nabla_{\eta^i} V^{i,\hat{\pibf}}(s') \right) \\
    &= \sum_{x} \sum_{k=0}^{\infty} P(s_0\rightarrow x, k, \hat{\pibf}) \gamma^k \sum_{\hat{\abf}} \hat{\pi}^i \sum_{j \neq i} (\nabla_{\eta^i} \hat{\pi}^j) \hat{\pi}^{\neq j,i} Q^{i, \hat{\pibf}} \\
    &= \sum_{\hat{s}} d^{\hat{\pibf}}(\hat{s}) \sum_{\hat{\abf}} \hat{\pi}^i \sum_{j \neq i} \hat{\pi}^j (\nabla_{\eta^i} \log \hat{\pi}^j) \hat{\pi}^{\neq j,i} Q^{i, \hat{\pibf}} \\
    &= \sum_{\hat{s}} d^{\hat{\pibf}}(\hat{s}) \sum_{\hat{\abf}} \hat{\pi}^i \sum_{j \neq i} (\nabla_{\eta^i} \log \hat{\pi}^j) \hat{\pi}^{-i} Q^{i, \hat{\pibf}} \\
    &= \sum_{\hat{s}} d^{\hat{\pibf}}(\hat{s}) \sum_{\hat{\abf}} \hat{\pibf} \sum_{j \neq i} (\nabla_{\eta^i} \log \hat{\pi}^j) Q^{i, \hat{\pibf}} = \Ebb_{\hat{\pibf}} \left[ \sum_{j \neq i} (\nabla_{\eta^i} \log \hat{\pi}^j) Q^{i, \hat{\pibf}} \right]
\end{align*}
Hence descending a stochastic estimate of this gradient is equivalent to minimizing the loss in \eqref{eq:eta-loss}.

\section{Experiments}

\subsection{Environment details}

This section provides more details on each experimental setup.

\textbf{IPD.}
We used the same definition of observation, action, and rewards as \citet{foerster2018learning}.
Each environment step is one round of the matrix game.
Each agent observes the joint action taken by both agents at the previous step, along with an indicator for the first round of each episode.
We trained for 60k episodes, each with 5 environments steps, which gives the same total number of environment steps used by LOLA \citep{foerster2018learning}.

\textbf{Escape Room.}
Each agent observes all agents' positions and can move among the three available states: lever, start, and door.
At every time step, all agents commit to and disclose their chosen actions, compute the incentives based on their observations of state and others' actions (only for LIO and augmented baselines that allow incentivization), and receive the sum of extrinsic rewards and incentives (if any).
LIO and augmented baselines also observe the cumulative incentives given to the other agents within the current episode.
An agent's individual reward is zero for staying at the current state, -1 for movement away from its current state if fewer than $M$ agents move to (or are currently at) the lever, and +10 for moving to (or staying at) the door if $\geq M$ agents pull the lever.
Each episode terminates when an agent successfully exits the door, or when 5 time steps elapse.

\textbf{Cleanup.}
We built on a version of an open-source implementation \citep{vinitsky2020}.
The environment settings for 7x7 and 10x10 maps are given in \Cref{tab:cleanup-settings}.
To focus on the core aspects of the common-pool resource problem, we removed rotation actions, set the orientation of all agents to face ``up'', and disabled their ``tagging beam'' (which, if used, would remove a tagged agent from the environment for a number of steps).
These changes mean that an agent must move to the river side of the map to clear waste successfully, as it cannot simply stay in the apple patch and fire its cleaning beam toward the river.
Acting cooperatively as such would allow other agents to collect apples, and hence our setup increases the difficulty of the social dilemma.
Each agent receives an egocentric normalized RGB image observation that spans a sufficiently large area such that the entire map is observable by that agent regardless of its position.
The cleaning beam has length 5 and width 3.
For LIO and the AC-c baseline, which have a separate module that observes other agents' actions and outputs real-valued incentives, we let that module observe a multi-hot vector that indicates which agent(s) used their cleaning beam.

\begin{table}[h]
    \centering
    \caption{Environment settings in Cleanup}
    \label{tab:cleanup-settings}
    \begin{tabular}{lrr}
        \toprule
        Parameter & 7x7 & 10x10 \\
        \midrule
        appleRespawnProbability & 0.5 & 0.3 \\
        thresholdDepletion & 0.6 & 0.4 \\ 
        thresholdRestoration & 0.0 & 0.0 \\
        wasteSpawnProbability & 0.5 & 0.5 \\
        view\_size & 4 & 7 \\
        max\_steps & 50 & 50 \\
        \bottomrule
    \end{tabular}
\end{table}

\subsection{Implementation}
\label{app:implementation}

This subsection provides more details on implementation of all algorithms used in experiments.
We use fully-connected neural networks for function approximation in the IPD and ER, and convolutional networks to process image observations in Cleanup.
The policy network has a softmax output for discrete actions in all environments.
Within each environment, all algorithms use the same neural architecture unless stated otherwise.
We applied the open-source implementation of LOLA \citep{foerster2018learning} to ER.
We use an exploration lower bound $\epsilon$ that maps the learned policy $\pi$ to a behavioral policy $\tilde{\pi}(a|s) = (1 - \epsilon) \pi(a|s) + \epsilon / |\Acal|$, with $\epsilon$ decaying linearly from $\epsilon_{\text{start}}$ to $\epsilon_{\text{end}}$ by $\epsilon_{\text{div}}$ episodes.
We use discount factor $\gamma = 0.99$.
We use gradient descent for policy optimization, the Adam optimizer \citep{kingma2014adam} for training value functions (in Cleanup), and Adam optimizer for LIO's incentive function.

The augmented policy gradient and actor-critic baselines, labeled as PG-c and AC-c, which have continuous ``give-reward'' actions in addition to regular discrete actions, are trained as follows.
These baselines have an augmented action space $\Acal \times \Rbb^{N-1}$ and learns a factorized policy $\pi(a_d,a_r|o) \defeq \pi(a_d|o) \pi(a_r|o)$, where $a_d \in \Acal$ is a regular discrete action and $a_r \in \Rbb^{N-1}$ is the reward given to the other $N-1$ agents.
The factor $\pi(a_d|o)$ is a standard categorical distribution conditioned on observation.
The factor $\pi(a_r|o)$ is defined via an element-wise sigmoid $\sigma(\cdot)$ applied to samples from a multivariate diagonal Gaussian, so that $\pi(a_r|o)$ is bounded.
Specifically, we let $u \sim \Ncal(f_{\eta}(o), \mathbf{1})$, where $f_{\eta}(o) \colon \Ocal \mapsto \Rbb^{N-1}$ is a neural network with parameters $\eta$, and let $a_r = R_{\text{max}} \sigma(u)$.
By the change of variables formula, $\pi(a_r|o)$ has density $\pi(a_r|o) = \Ncal(\mu_{\eta},\mathbf{1}) \prod_{i=1}^{N-1} (\mathrm{d}a_r[i]/\mathrm{d}u[i])^{-1}$, which can be used to compute the log-likelihood of $\pi(a_d,a_r|o)$ in the policy gradient.

Let $\beta$ denote the coefficient for entropy of the policy, $\alpha_{\theta}$ the policy learning rate, $\alpha_{\eta}$ the incentive learning rate, $\alpha_{\phi}$ the critic learning rate, and $R_a$ the value of the discrete ``give-reward'' action.

\textbf{IPD.}
The policy network and the incentive function in LIO have two hidden layers of size 16 and 8.

\begin{table}[h]
    \centering
    \caption{Hyperparameters in IPD.}
    \label{tab:hyperparam-ipd}
    \begin{tabular}{lrlr}
        \toprule
        Parameter & Value & Parameter & Value \\
        \midrule
        $\beta$ & 0.1 & $\alpha_{\theta}$ & 1e-3 \\
        $\epsilon_{\text{start}}$ & 1.0 & $\alpha_{\eta}$ & 1e-3 \\
        $\epsilon_{\text{end}}$ & 0.01 & $\alpha$ & 0 \\
        $\epsilon_{\text{div}}$ & 5000 & $R_{\text{max}}$ & 3.0 \\
        \bottomrule
    \end{tabular}
\end{table}

\textbf{ER.}
The policy network has two hidden layers of size 64 and 32.
LIO's incentive function has two hidden layers of size 64 and 16.
We use a separate Adam optimizer for the cost part of the incentive function's objective \eqref{eq:reward-objective}, with learning rate 1e-4, with $\alpha_{\eta} =$ 1e-3, and set $\alpha=1.0$.
Exploration and learning rate hyperparameters were tuned for each algorithm via coordinate ascent, searching through $\epsilon_{\text{start}}$ in [0.5, 1.0], $\epsilon_{\text{end}}$ in [0.05, 0.1, 0.3], $\epsilon_{\text{div}}$ in [100, 1000], $\beta$ in [0.01, 0.1],  $\alpha_{\theta}$, $\alpha_{\eta}$, and $\alpha_{\text{cost}}$ in [1e-3, 1e-4].
LOLA performed best with learning rate 0.1 and $R_a=2.0$, but it did not benefit from additional exploration.
LIO and PG-c have $R_{\text{max}} = 2.0$.
PG-d used $R_a = 2.0$.

\begin{table}[h]
    \centering
    \caption{Hyperparameters in Escape Room.}
    \label{tab:hyperparam-er}
    \begin{tabular}{lrrrrrrrr}
        \toprule
        \multicolumn{4}{r}{$N=2$} & \multicolumn{5}{c}{$N=3$}\\
        \cmidrule(r){2-5}
        \cmidrule(r){6-9}
        Parameter & LIO & PG & PG-d & PG-c & LIO & PG & PG-d & PG-c \\
        \midrule
        $\beta$ & 0.01 & 0.01 & 0.01 & 0.1 & 0.01 & 0.01 & 0.01 & 0.1 \\
        $\epsilon_{\text{start}}$ & 0.5 & 0.5 & 0.5 & 1.0 & 0.5 & 0.5 & 0.5 & 1.0 \\
        $\epsilon_{\text{end}}$ & 0.1 & 0.05 & 0.05 & 0.1 & 0.3 & 0.05 & 0.05 & 0.1 \\
        $\epsilon_{\text{div}}$ & 1e3 & 1e2 & 1e2 & 1e3 & 1e3 & 1e2 & 1e2 & 1e3 \\
        $\alpha_{\theta}$ & 1e-4 & 1e-4 & 1e-4 & 1e-3 & 1e-4 & 1e-4 & 1e-4 & 1e-3\\
        \bottomrule
    \end{tabular}
\end{table}


\textbf{Cleanup.}
All algorithms are based on actor-critic for policy optimization, whereby each agent $j$'s policy parameter $\theta^j$ is updated via
\begin{align}\label{eq:actor-critic}
    \hat{\theta}^j \leftarrow \theta^j + \Ebb_{\pibf} \left[ \nabla_{\theta^j} \log \pi_{\theta^j}(a^j|o^j) \left( r^j + \gamma V_{\phi^j}(s') - V_{\tilde{\phi}^j}(s) \right) \right] \, ,
\end{align}
and the critic parameter $\phi^j$ is updated by minimizing the temporal difference loss
\begin{align}
    L(\phi^j) &= \Ebb_{s,s' \sim \pibf} \Bigl[ \bigl(r^j + \gamma V_{\tilde{\phi}^j}(s') - V_{\phi^j}(s) \bigr)^2 \Bigr]
\end{align}
The target network \citep{mnih2015human} parameters $\tilde{\phi}^j$ are updated via $\tilde{\phi}^j \leftarrow \tau \phi^j + (1 -\tau) \tilde{\phi}^j$ with $\tau = 0.01$.

The policy and value networks have an input convolutional layer with 6 filters of size [3, 3], stride [1, 1], and ReLU activation.
The output of convolution is flattened and passed through two fully-connected (FC) hidden layers both of size 64.
The policy output is a softmax over discrete actions; the value network has a linear scalar output.
LIO's incentive function uses the same input convolutional layer, except that its output is passed through the first FC layer, concatenated with its observation of other agents' actions, then passed through the second FC layer and finally to a linear output layer.
Inequity Aversion agents \citep{hughes2018inequity} have an additional 1D vector observation of all agents' temporally smoothed rewards---this is concatenated with the output of the first FC hidden layer and sent to the second FC layer.
Entropy coefficient was held at 0.1 for all methods.

LIO and AC-c have $R_{\text{max}} = 2.0$.
AC-d used $R_a = 2.0$.
Inequity aversion agents have disadvantageous aversion coefficient value 0, advantageous aversion coefficient value 0.05, and temporal smoothing parameter $\lambda = 0.95$.
We use critic learning rate $\alpha_{\phi}=10^{-3}$ for all methods.
LIO used $\alpha_{\eta}=$1e-3 and cost coefficient $\alpha=10^{-4}$.
Exploration and learning rate hyperparameters were tuned for each algorithm via coordinate ascent, searching through $\epsilon_{\text{start}}$ in [0.5, 1.0], $\epsilon_{\text{end}}$ in [0.05, 0.1], $\epsilon_{\text{div}}$ in [100, 1000, 5000], $\alpha_{\theta}$, $\alpha_{\eta}$, and $\alpha_{\text{cost}}$ in [1e-3, 1e-4].

\begin{table}[h]
    \centering
    \caption{Hyperparameters in Cleanup.}
    \label{tab:hyperparam-cleanup}
    \begin{tabular}{lrrrrrrrrrr}
        \toprule
        \multicolumn{4}{r}{7x7} & \multicolumn{5}{r}{10x10}\\
        \cmidrule(r){2-6}
        \cmidrule(r){7-11}
        Parameter & LIO & AC & AC-d & AC-c & IA & LIO & AC & AC-d & AC-c & IA \\
        \midrule
        $\epsilon_{\text{start}}$ & 0.5 & 0.5 & 0.5 & 0.5 & 0.5 & 0.5 & 0.5 & 0.5 & 0.5 & 0.5 \\
        $\epsilon_{\text{end}}$ & 0.05 & 0.05 & 0.05 & 0.05 & 0.05 & 0.05 & 0.05 & 0.05 & 0.05 & 0.05 \\
        $\epsilon_{\text{div}}$ & 100 & 100 & 100 & 100 & 1000 & 1000 & 5000 & 1000 & 1000 & 5000 \\
        $\alpha_{\theta}$ & 1e-4 & 1e-3 & 1e-4 & 1e-4 & 1e-3 & 1e-4 & 1e-3 & 1e-3 & 1e-3 & 1e-3 \\
        \bottomrule
    \end{tabular}
\end{table}


\section{Additional results}

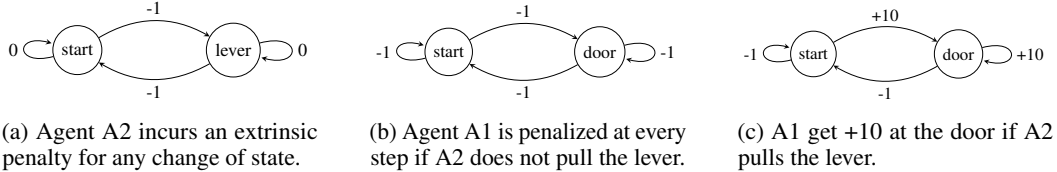
\begin{figure*}[t]
    \centering
    \begin{subfigure}[t]{0.3\linewidth}
    \centering
    \resizebox{1.0\textwidth}{!}{%
    \begin{tikzpicture}
        \node[state] at (-1,0) (s0) {start};
        \node[state, right of=s0] at (1,0) (s1) {lever};
        \draw (s0) edge[loop left] node{0} (s0)
            (s1) edge[loop right] node{0} (s1)
            (s0) edge[bend left, above] node{-1} (s1)
            (s1) edge[bend left, below] node{-1} (s0)
        ;
    \end{tikzpicture}
    }
    \caption{Agent A2 incurs an extrinsic penalty for any change of state.}
    \label{fig:a2-mdp}
    \end{subfigure}
    \hfill
    \begin{subfigure}[t]{0.3\linewidth}
    \centering
    \resizebox{1.0\textwidth}{!}{%
    \begin{tikzpicture}
        \node[state] at (-1,0) (s0) {start};
        \node[state, right of=s0] at (1,0) (s1) {door};
        \draw (s0) edge[loop left] node{-1} (s0)
            (s1) edge[loop right] node{-1} (s1)
            (s0) edge[bend left, above] node{-1} (s1)
            (s1) edge[bend left, below] node{-1} (s0)
        ;
    \end{tikzpicture}
    }
    \caption{Agent A1 is penalized at every step if A2 does not pull the lever.}
    \label{fig:a1-mdp1}
    \end{subfigure}
    \hfill
    \begin{subfigure}[t]{0.3\linewidth}
    \centering
    \resizebox{1.0\textwidth}{!}{%
    \begin{tikzpicture}
        \node[state] at (-1,0) (s0) {start};
        \node[state, right of=s0] at (1,0) (s1) {door};
        \draw (s0) edge[loop left] node{-1} (s0)
            (s1) edge[loop right] node{+10} (s1)
            (s0) edge[bend left, above] node{+10} (s1)
            (s1) edge[bend left, below] node{-1} (s0)
        ;
    \end{tikzpicture}
    }
    \caption{A1 get +10 at the door if A2 pulls the lever.}
    \label{fig:a1-mdp2}
    \end{subfigure}
    \centering
    \caption{Asymmetric \textit{Escape Room} game involving two agents, A1 and A2. (a) In the absence of incentives, A2's optimal policy is to stay at the start state and not pull the lever. (b) Hence A1 cannot exit the door and is penalized at every step. (c) A1 can receive positive reward if it learns to incentivize A2 to pull the lever. Giving incentives is not an action depicted here.}
    \label{fig:asymmetric-room}
\end{figure*}

\subsection{Asymmetric Escape Room}
\label{app:asymmetric}

We conducted additional experiments on an asymmetric version of the Escape Room game between two learning agents (A1 and A2) as shown in \Cref{fig:asymmetric-room}. 
A1 gets +10 extrinsic reward for exiting a door and ending the game (\Cref{fig:a1-mdp2}), but the door can only be opened when A2 pulls a lever; otherwise, A1 is penalized at every time step (\Cref{fig:a1-mdp1}).
The extrinsic penalty for A2 discourages it from taking the cooperative action (\Cref{fig:a2-mdp}).
The global optimum combined reward is +9, and it is impossible for A2 to get positive extrinsic reward.
Due to the asymmetry, A1 is the reward-giver and A2 is the reward recipient for methods that allow incentivization.
Each agent observes both agents' positions, and can move between the two states available to itself.
We allow A1 to observe A2's current action before choosing its own action, which is necessary for methods that learn to reward A2's cooperative actions.
We use a standard policy gradient for A2 unless otherwise specified.

In addition to the baselines described for the symmetric case---namely, policy gradient (PG-rewards) and LOLA with discrete ``give-reward'' actions---we also compare with a two-timescale method, labeled \textbf{2-TS}.
A 2-TS agent has the same augmented action space as the PG-rewards baseline, except that it learns over a longer time horizon than the reward recipient.
Each ``epoch'' for the 2-TS agent spans multiple regular episodes of the recipient, during which the 2-TS agent executes a fixed policy.
The 2-TS agent only caries out a learning update using a final terminal reward, which is the average extrinsic rewards it gets during \textit{test} episodes that are conducted at the end of the epoch.
Performance on test episodes serve as a measure of whether correct reward-giving actions were taken to influence the recipient's learning during the epoch.
To our knowledge, 2-TS is a novel baseline but has key limitations:
the use of two timescales only applies to the asymmetric 2-player game, and requires fast learning by the reward-recipient, chosen to be a tabular Q-learning, to avoid intractably long epochs.

\begin{figure*}[t]
\centering
\begin{subfigure}[t]{0.24\linewidth}
    \centering
    \includegraphics[width=1.0\linewidth]{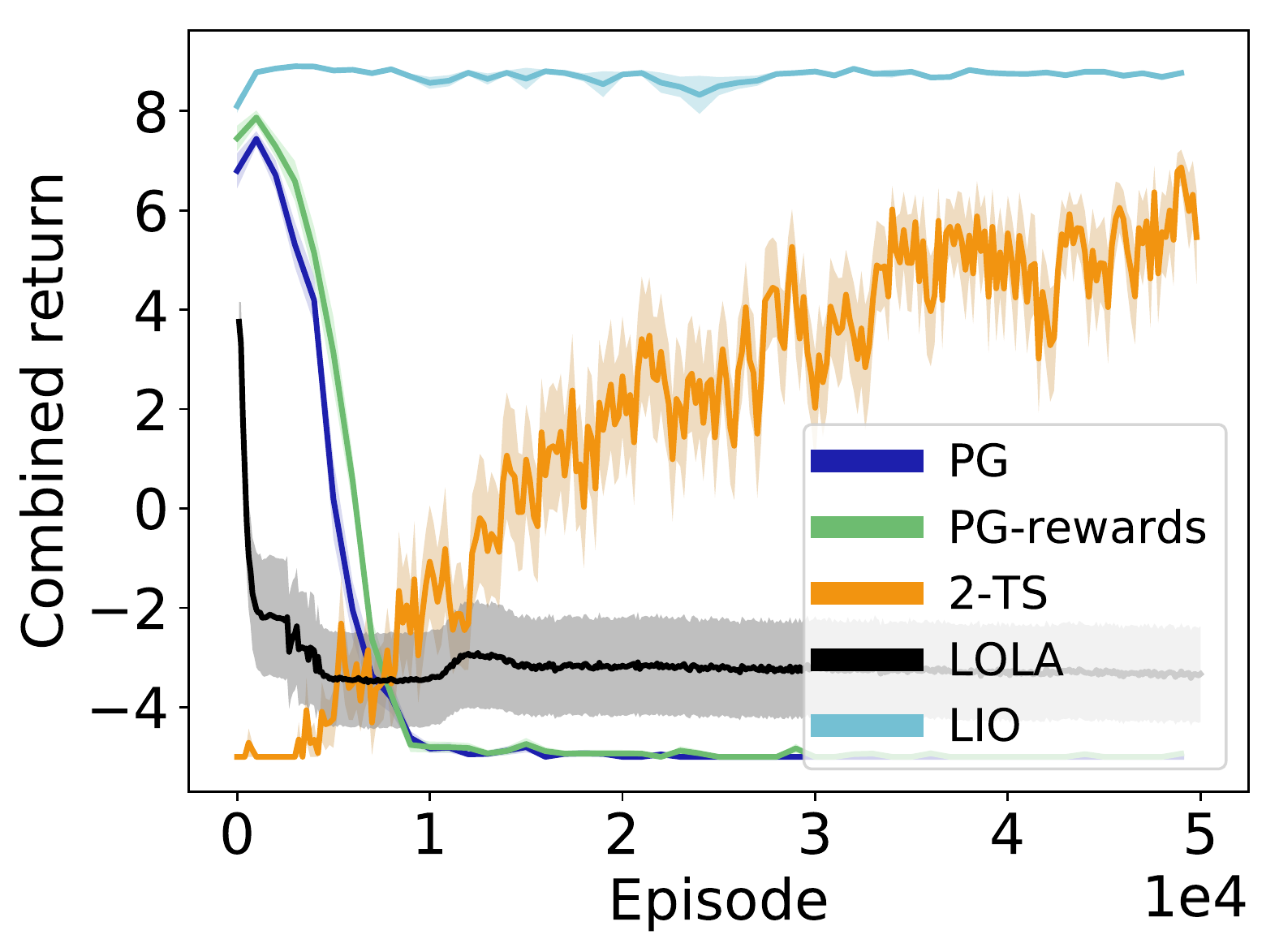}
    \caption{Sum of agent rewards}
    \label{fig:asymmetric-all}
\end{subfigure}
\hfill
\begin{subfigure}[t]{0.24\linewidth}
    \centering
    \includegraphics[width=1.0\linewidth]{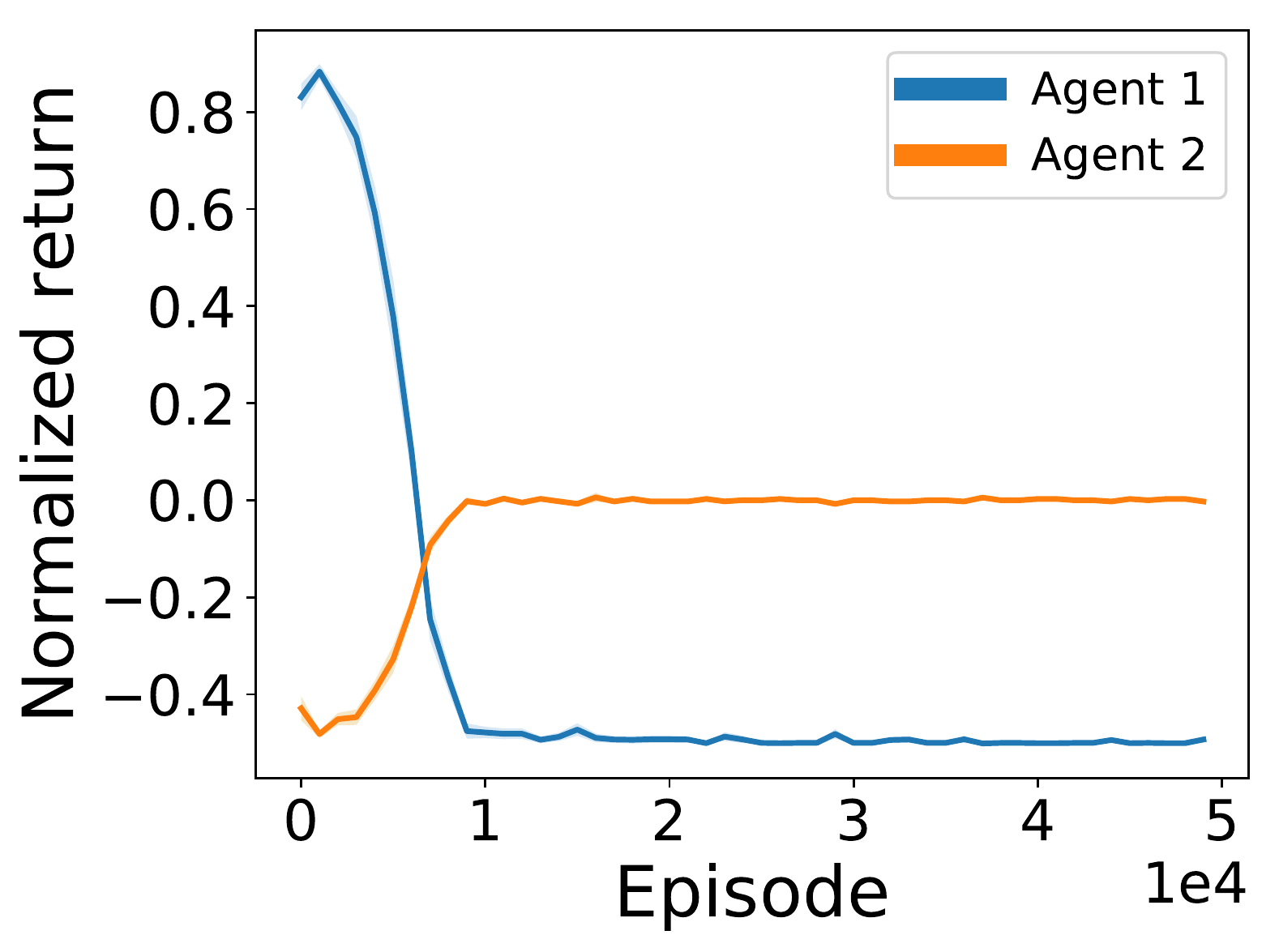}
    \caption{Two PG agents}
    \label{fig:asymmetric-pg}
\end{subfigure}
\hfill
\begin{subfigure}[t]{0.24\linewidth}
    \centering
    \includegraphics[width=1.0\linewidth]{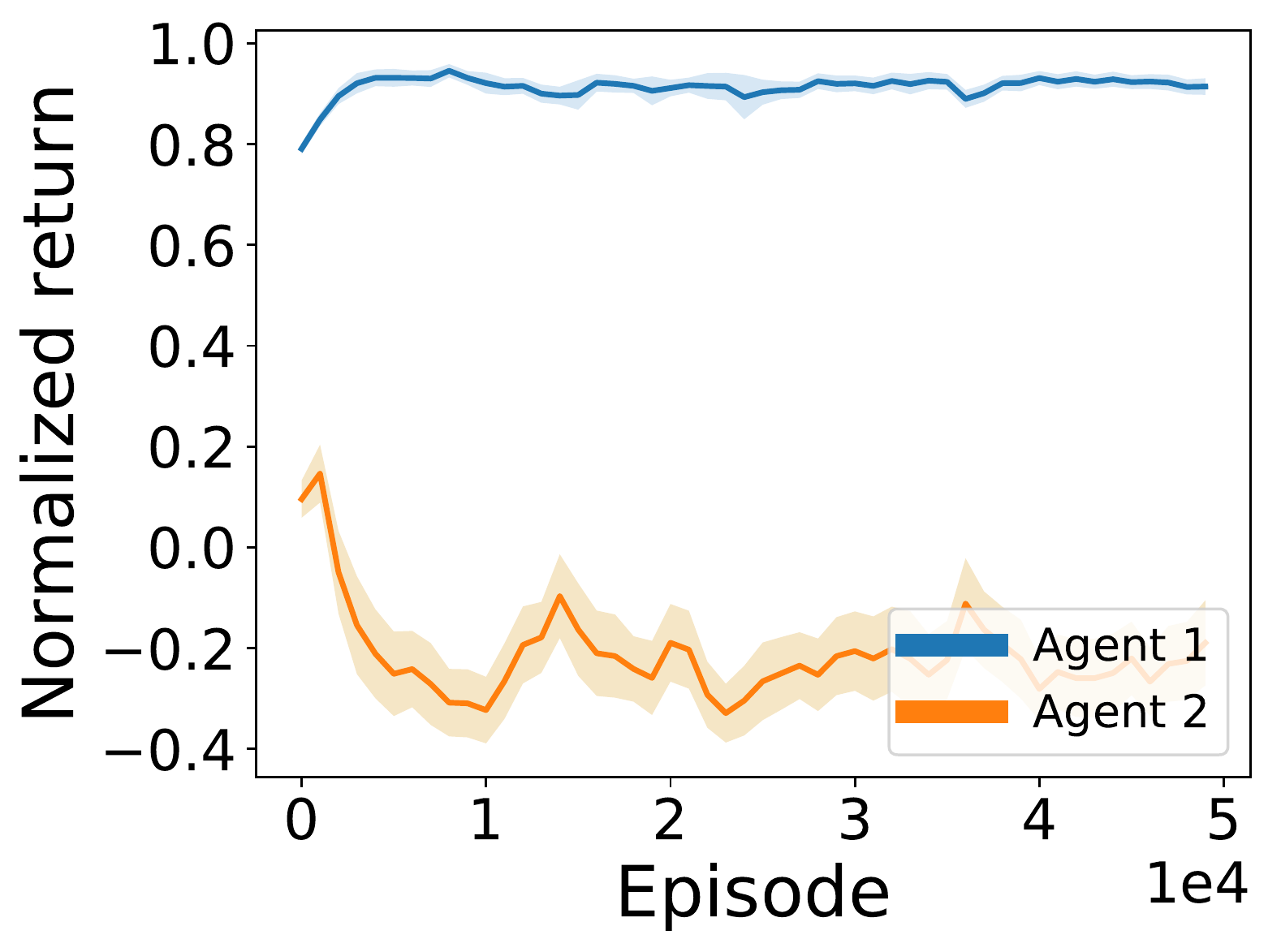}
    \caption{LIO (A1) and PG agent (A2)}
    \label{fig:asymmetric-lio}
\end{subfigure}
\hfill
\begin{subfigure}[t]{0.24\linewidth}
    \centering
    \includegraphics[width=1.0\linewidth]{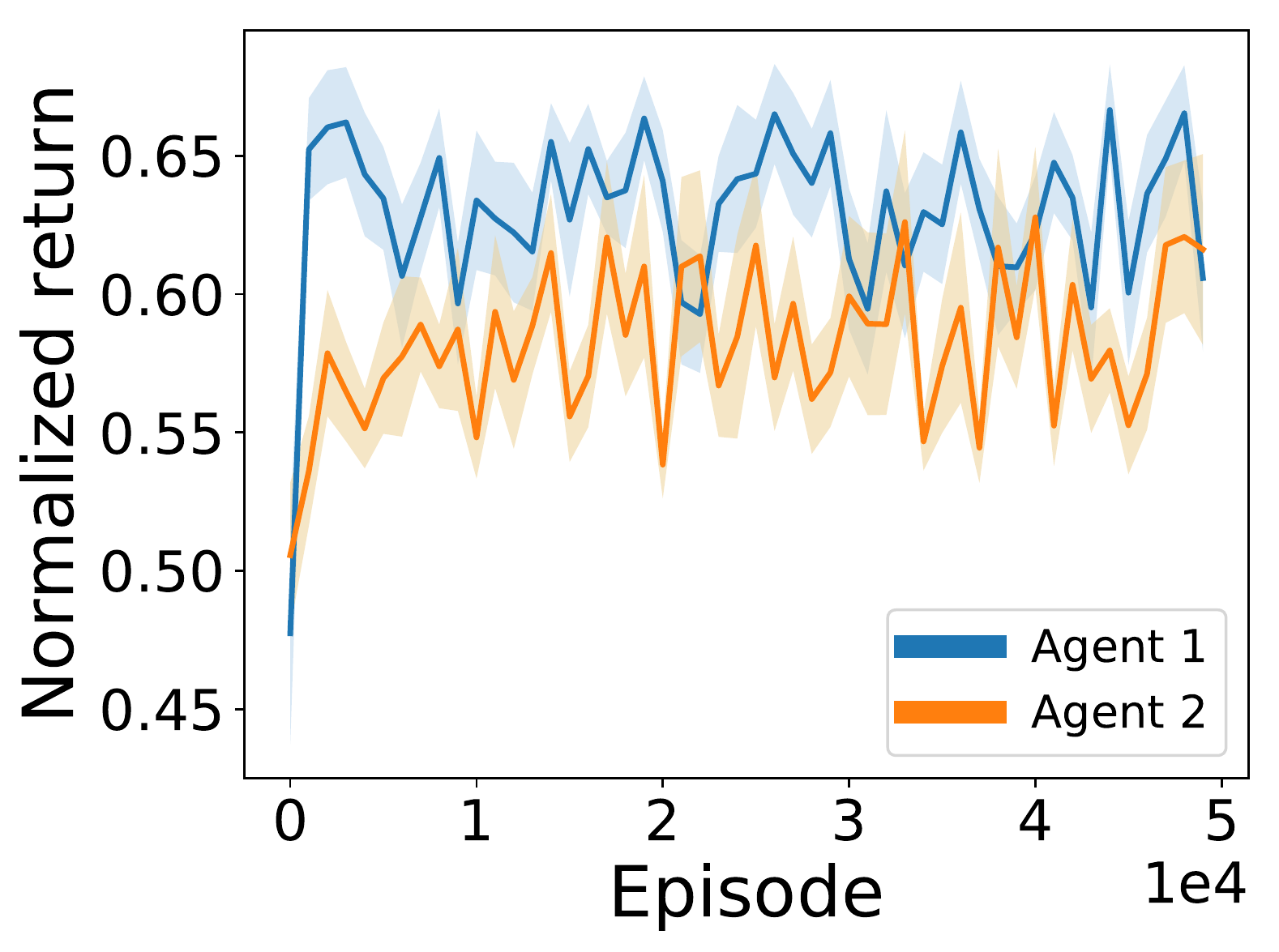}
    \caption{1-episode LIO and PG agent}
    \label{fig:asymmetric-lio-1ep}
\end{subfigure}
\caption{Results in asymmetric 2-player Escape Room. (a) LIO (paired with PG) converges rapidly to the global optimum, 2-TS (paired with tabular Q-learner) converges slower, while policy gradient baselines could not cooperate. (b) Two PG agents cannot cooperate, as A2 converges to ``do-nothing''. (c) A LIO agent (A1) attains near-optimum reward by incentivizing a PG agent (A2). (d) 1-episode LIO has larger variance and lower performance. Normalization factors are 1/10 (A1) and 1/2 (A2).}
\label{fig:asymmetric-2agent}
\end{figure*}

\Cref{fig:asymmetric-2agent} shows the sum of both agents' rewards for all methods on the asymmetric 2-player game, as well as agent-specific performance for policy gradient and LIO, across training episodes.
A LIO reward-giver agent paired with a policy gradient recipient converges rapidly to a combined return near $9.0$ (\Cref{fig:asymmetric-all}), which is the global maximum, while both PG and PG-rewards could not escape the global minimum for A1.
LOLA paired with a PG recipient found the cooperative solution in two out of 20 runs; this suggests the difficulty of using a fixed incentive value to conduct opponent shaping via discrete actions.
The 2-TS method is able to improve combined return but does so much more gradually than LIO, because an epoch consists of many base episodes and it depends on a highly delayed terminal reward.
\Cref{fig:asymmetric-pg} for two PG agents shows that A2 converges to the policy of not moving (reward of 0), which results in A1 incurring penalties at every time step.
In contrast, \Cref{fig:asymmetric-lio} verifies that A1 (LIO) receives the large extrinsic reward (scaled by 1/10) for exiting the door, while A2 (PG) has average normalized reward above -0.5 (scaled by 1/2), indicating that it is receiving incentives from A1.
Average reward of A2 (PG) is below 0 because incentives given by A1 need not exceed 1 continually during training---once A2's policy is biased toward the cooperative action in early episodes, its decaying exploration rate means that it may not revert to staying put even when incentives do not overcome the penalty for moving.
\Cref{fig:asymmetric-lio-1ep} shows results on a one-episode version of LIO where the same episode is used for both policy update and incentive function updates, with importance sampling corrections.
This version performs significantly lower  for A1 and gives more incentives than is necessary to encourage A2 to move.
It demonstrates the benefit of learning the reward function using a separate episode from that in which it is applied.


\subsection{Symmetric Escape Room}
\label{app:symmetric}

\Cref{fig:app-er} shows total reward (extrinsic + received - given incentives), counts of ``lever'' and ``door'' actions, and received incentives in one training run each for ER(2,1) and ER(3,2).
In \Cref{fig:symmetric-2-meas}, A1 becomes the winner and A2 the cooperator.
It is not always necessary for A1 to give rewards.
The fact that LIO models the learning updates of recipients may allow it to find that reward-giving is unnecessary during some episodes when the recipient's policy is sufficiently biased toward cooperation.
In \Cref{fig:symmetric-3-meas}, A3 converges to going to the door, as it incentives A1 and A2 to pull the lever.

\begin{figure*}[h]
\centering
\begin{subfigure}[t]{0.3\linewidth}
    \centering
    \includegraphics[width=1.0\linewidth]{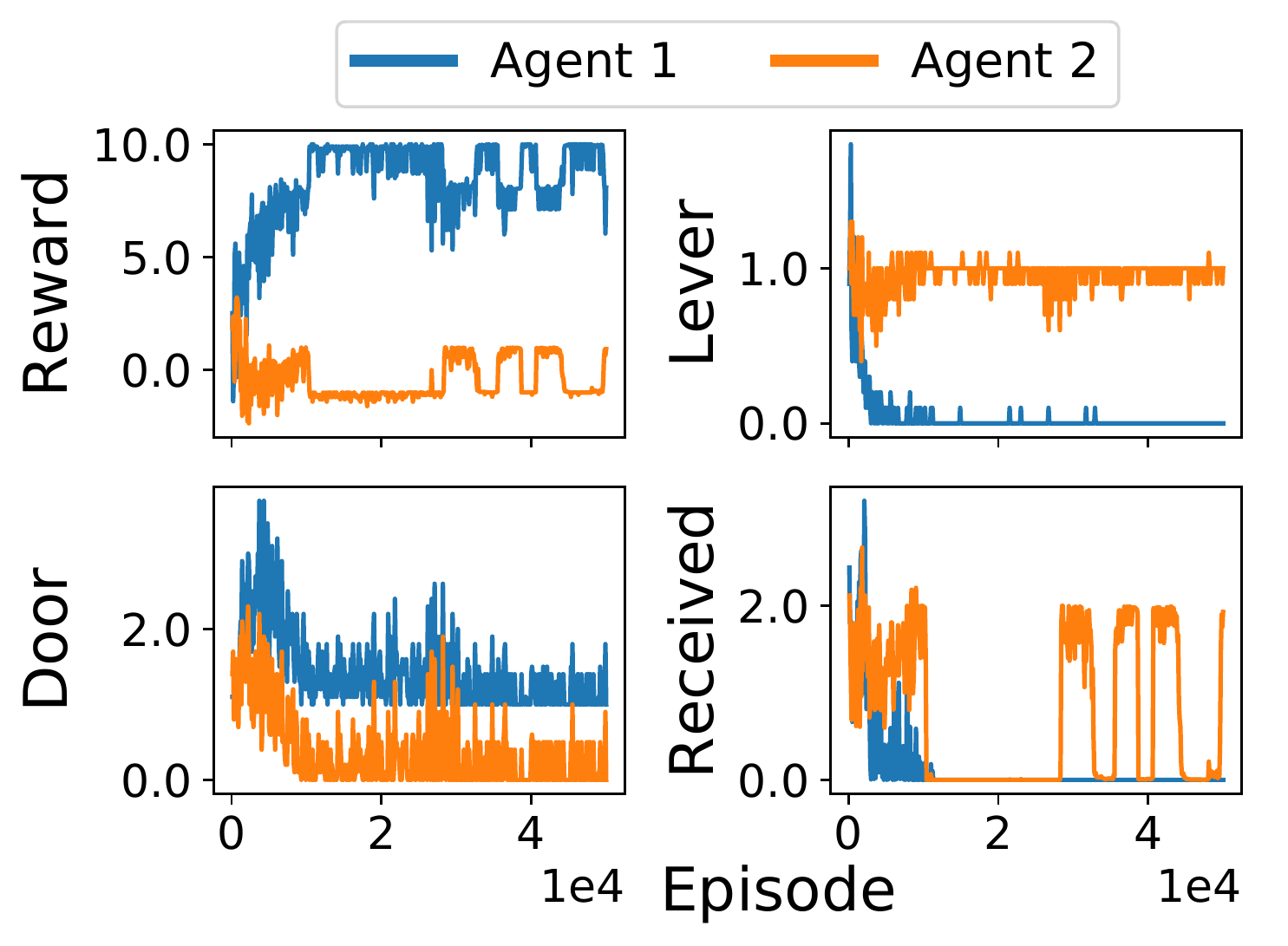}
    \caption{ER(2,1)}
    \label{fig:symmetric-2-meas}
\end{subfigure}
\hfill
\begin{subfigure}[t]{0.3\linewidth}
    \centering
    \includegraphics[width=1.0\linewidth]{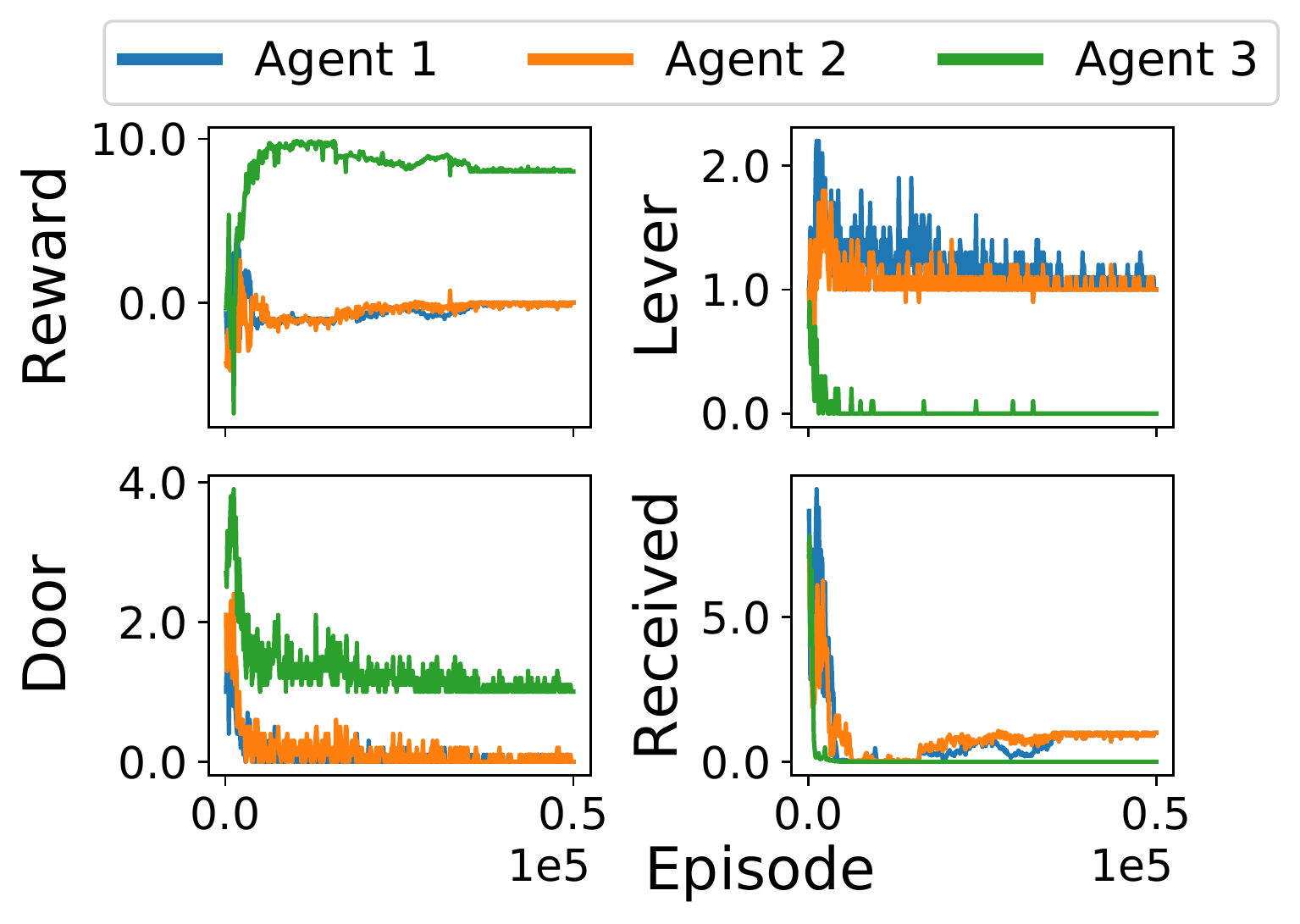}
    \caption{ER(3,2)}
    \label{fig:symmetric-3-meas}
\end{subfigure}
\hfill
\begin{subfigure}[t]{0.3\linewidth}
    \centering
    \includegraphics[width=1.0\linewidth]{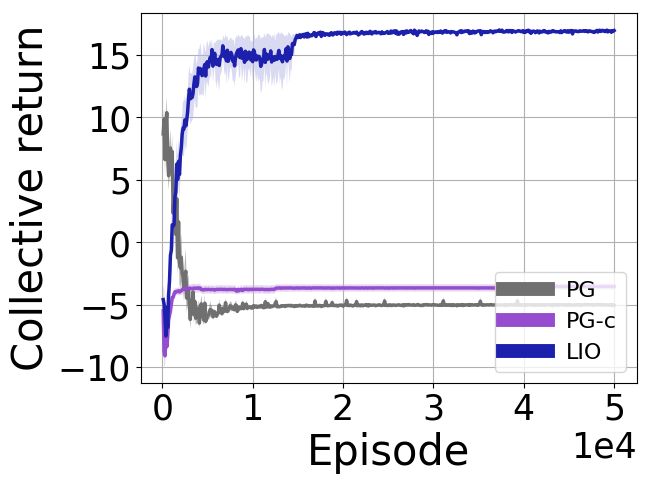}
    \caption{ER(5,3)}
    \label{fig:er_n5}
\end{subfigure}
\caption{(a,b) Individual actions and incentives in ER(2,1) and ER(3,2). (c) LIO converges to the global optimum in ER(5,3).}
\label{fig:app-er}
\end{figure*}

\subsection{Cleanup}

\begin{figure*}[t]
\centering
\begin{subfigure}[t]{0.2\linewidth}
    \centering
    \includegraphics[width=1.0\linewidth]{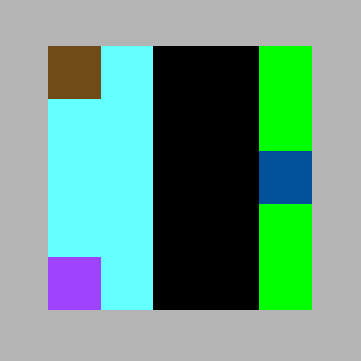}
    \caption{Division of labor}
    \label{fig:cleanup-small-lio-behavior}
\end{subfigure}
\hfill
\begin{subfigure}[t]{0.2\linewidth}
    \centering
    \includegraphics[width=1.0\linewidth]{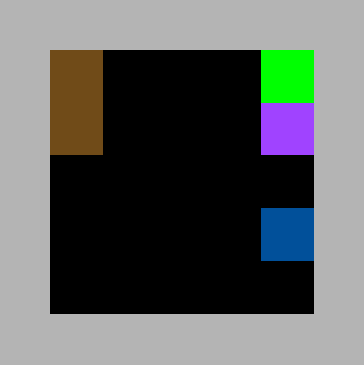}
    \caption{AC agents compete}
    \label{fig:cleanup-small-ac-behavior}
\end{subfigure}
\hfill
\begin{subfigure}[t]{0.28\linewidth}
    \centering
    \includegraphics[width=1.0\linewidth]{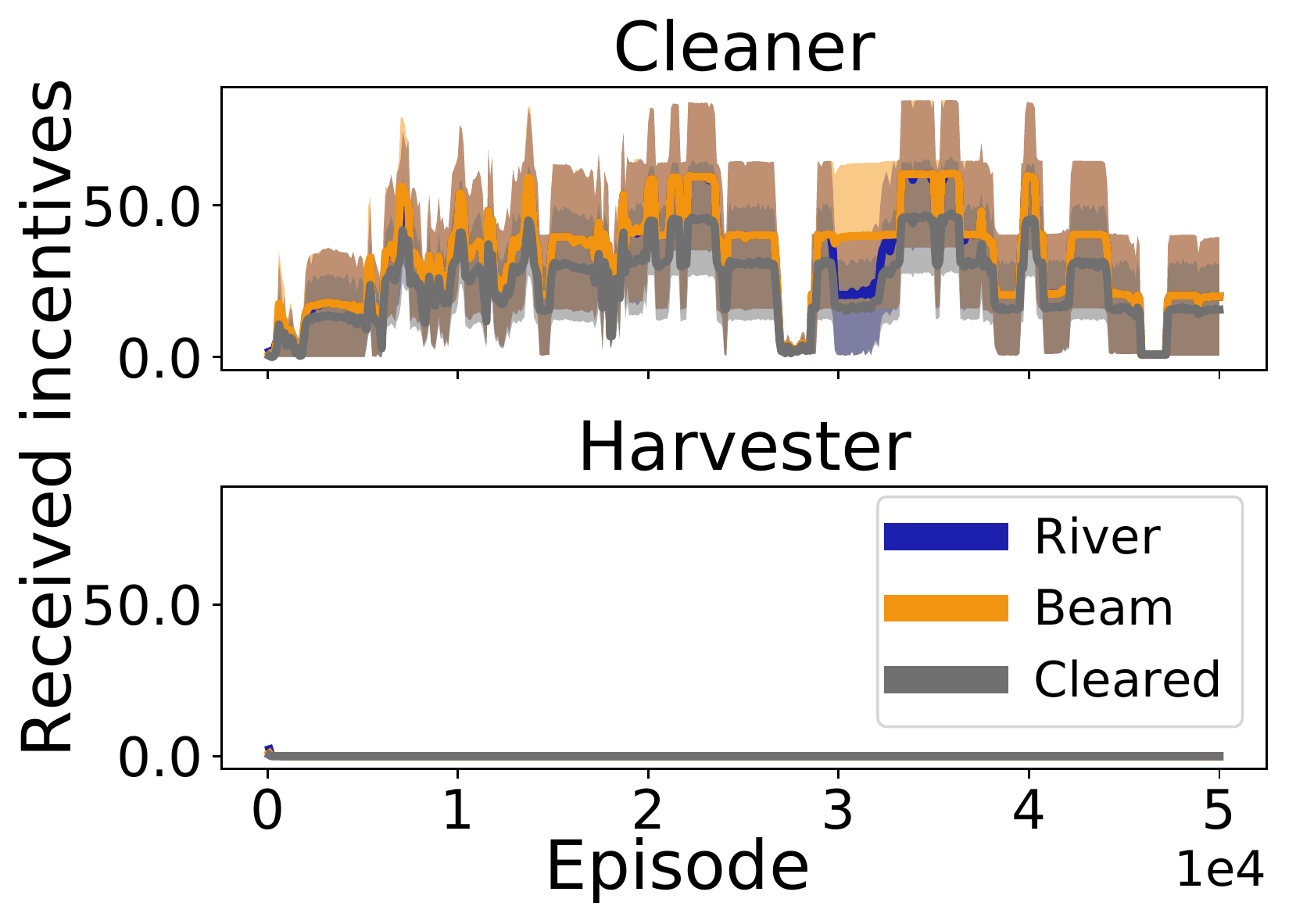}
    \caption{Cleaner's incentives}
    \label{fig:cleanup-small-received}
\end{subfigure}
\hfill
\begin{subfigure}[t]{0.28\linewidth}
    \centering
    \includegraphics[width=1.0\linewidth]{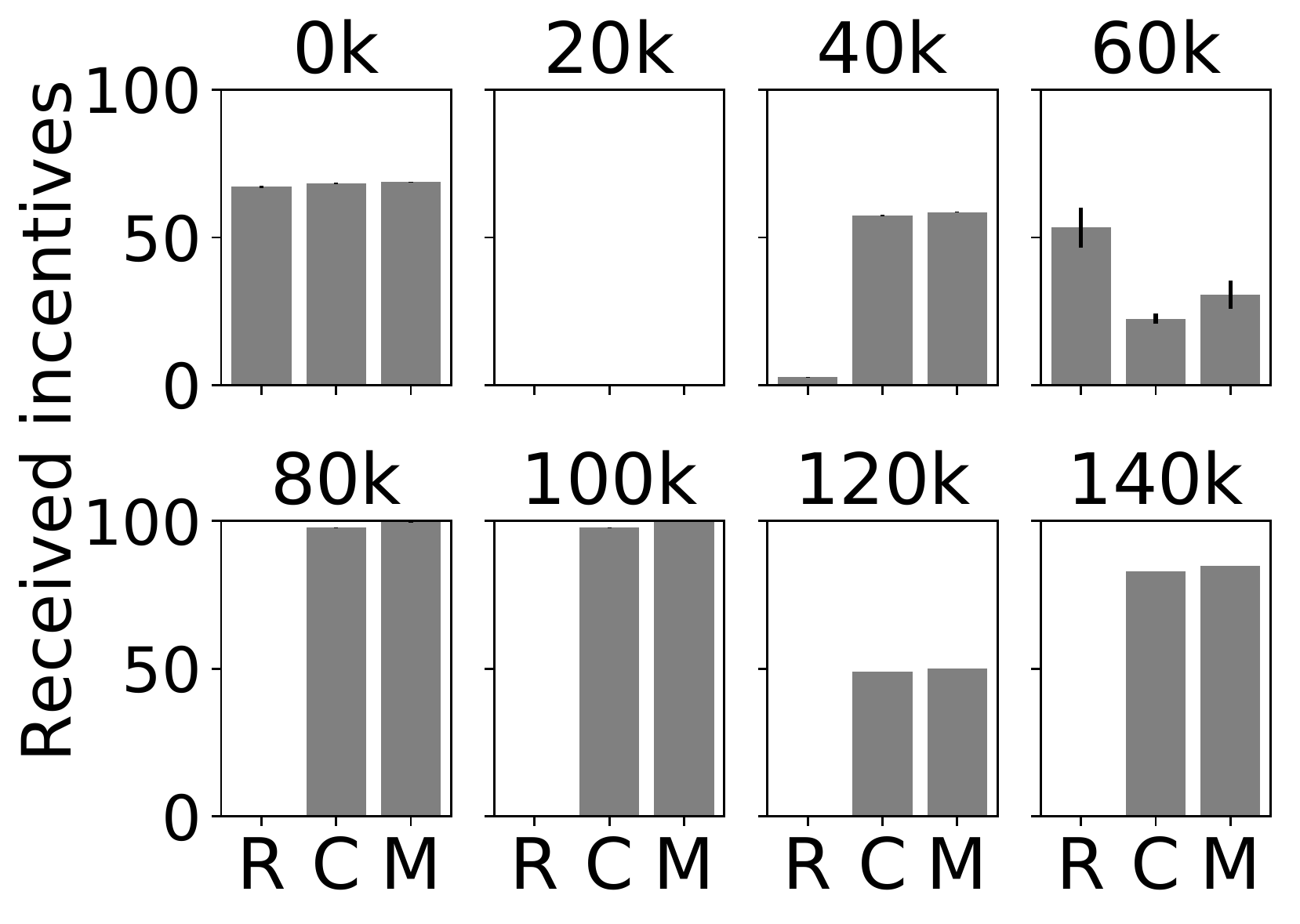}
    \caption{10x10 map}
    \label{fig:cleanup-10x10-measure}
\end{subfigure}
\caption{(a) In 7x7 Cleanup, one LIO agent learns to focus on cleaning waste, as it receives incentives from the other who only collects apple. 
(b) In contrast, AC agents compete for apples after cleaning. 
(c) Incentives received during training on 7x7 Cleanup. 
(d) Behavior of incentive function against scripted opponent policies on 10x10 map.}
\label{fig:app-cleanup}
\end{figure*}

\Cref{fig:cleanup-small-lio-behavior} is a snapshot of the division of labor found by two LIO agents, whereby the blue agent picks apples while the purple agent stays on the river side to clean waste.
The latter does so because of incentives from the former.
In contrast, \Cref{fig:cleanup-small-ac-behavior} shows a time step where two AC agents compete for apples, which is jointly suboptimal.
\Cref{fig:cleanup-small-received} shows the received incentives during training in the 7x7 map, for each of two LIO agents that were classified after training as a ``Cleaner'' or ``Harvester''.
\Cref{fig:cleanup-10x10-measure} shows the incentives given by a ``Harvester'' agent to three scripted agents during each training checkpoint.

Agents with hand-designed intrinsic rewards based on social influence \citep{jaques2019social} also outperform standard RL agents on Cleanup.
We can make an indirect comparison to \citep{jaques2019social} by noting that IA reaches a score around 250 by $1.6\times10^8$ steps \citep[Figure 3a]{hughes2018inequity}, which outperforms the score of 200 attained by Social Influence at $3\times10^8$ steps \citep[Figure 1a]{jaques2019social} in the original Cleanup map with 5 agents.
Hence, the fact that LIO outperforms IA in our experiments suggests that LIO compares favorably with Social Influence, provided that LIO uses the same RL algorithm as the latter for policy optimization.


\end{document}